\documentclass{article} 
\usepackage{iclr2027_conference,times}
\iclrfinalcopy

\usepackage{amsmath,amsfonts,bm}

\def\eqref#1{equation~\ref{#1}}

\def\1{\bm{1}}

\DeclareMathAlphabet{\mathsfit}{\encodingdefault}{\sfdefault}{m}{sl}
\SetMathAlphabet{\mathsfit}{bold}{\encodingdefault}{\sfdefault}{bx}{n}

\usepackage{hyperref}
\usepackage{url}
\usepackage{kotex}
\usepackage{booktabs}
\usepackage{multirow}
\usepackage{makecell}
\usepackage{graphicx}
\usepackage{amsmath}
\usepackage{amssymb}
\usepackage{algorithm}
\usepackage{algpseudocode}
\usepackage{capt-of}
\usepackage{subcaption}
\usepackage{graphicx}
\usepackage{booktabs}
\usepackage{capt-of}
\usepackage{adjustbox}
\usepackage[table]{xcolor}

\definecolor{ablBlue}{HTML}{2465E8}
\definecolor{ablPurple}{HTML}{923BE5}
\definecolor{ablOrange}{HTML}{D78B00}
\definecolor{ablRed}{HTML}{EF3340}
\definecolor{ablTeal}{HTML}{009D87}

\title{StreamTTO: Efficient Online Test-Time Optimization for Video Depth Completion}
\author{%
Minseok Seo\thanks{Equal contribution.}, 
Wonjun Lee\footnotemark[1], 
Jaehyuk Jang \&
Changick Kim\thanks{Corresponding author.} \\
Korea Advanced Institute of Science and Technology (KAIST) \\
Daejeon, Republic of Korea \\
\texttt{minseok.seo@kaist.ac.kr, dpenguin@kaist.ac.kr} \\
\texttt{jhyuk@kaist.ac.kr, changick@kaist.ac.kr}
}

\begin{document}

\maketitle
\lhead{}
\renewcommand{\headrulewidth}{0pt}
\begin{abstract}
Monocular depth foundation models generalize across diverse scenes, but recovering accurate metric depth consistent with a target sensor remains challenging under sensor variation and domain shift.
Test-time optimization (TTO) guided by sparse depth addresses this limitation, yet independent per-frame adaptation incurs substantial computational cost.
We propose StreamTTO, an online test-time optimization framework for video depth completion that reuses visual features and adaptation state.
A frozen feature extractor with causal temporal attention incorporates past visual context into features cached for repeated decoder optimization.
Causal Sliding-Window Optimization directly updates a shared depth decoder using cached features and sparse observations from recent frames, requiring only two passes over the window per incoming frame after initialization.
Reusing recent observations and the adapted decoder substantially reduces optimization steps.
We also introduce MS-Depth, comprising approximately 1.45M synchronized RGB--LiDAR frames from 40 hours of recordings.
Its uninterrupted sequences capture transitions among indoor, underground, and outdoor environments under daytime and nighttime conditions, enabling evaluation of adaptation to changes in illumination, scene structure, and depth range.
Experiments on KITTI, Bonn, NYUv2 Raw, TUM RGB-D, DDAD, and MS-Depth demonstrate competitive depth completion accuracy.
On MS-Depth, StreamTTO achieves average processing speeds of approximately 15~FPS with VGGT and 25~FPS with MoGe-2 at an input resolution of $518 \times 392$.
\end{abstract}

\section{Introduction}
Accurate metric depth is important for autonomous driving, robotic navigation, and 3D reconstruction at physical scale~\citep{rajapaksha2024deep,geiger2013vision}.
Systems often obtain this information using range sensors such as LiDAR, but hardware and acquisition constraints limit measurement coverage and density~\citep{ma2018sparse,ma2019self}.
Depth completion fills these gaps by estimating dense depth from sparse observations.
RGB-guided methods learn to combine these measurements with visual cues about image boundaries and scene structure~\citep{ma2018sparse,cheng2018depth,tang2020learning,park2020non,zhang2023completionformer}.

Meanwhile, monocular depth foundation models have substantially improved dense depth estimation across diverse scenes from RGB images alone~\citep{yang2024depth2,piccinelli2025unidepthv2,wang2026moge}.
Beyond individual images, VGGT~\citep{wang2025vggt} and $\pi^3$~\citep{wang2026pi} reconstruct geometry from multiple views, while Video Depth Anything~\citep{chen2025video} incorporates temporal context for consistent video depth estimation.
StreamVGGT~\citep{zhuo2026streaming} further enables online geometry reconstruction through causal attention and cached representations of past frames.
Despite richer visual and temporal context, these models do not guarantee accurate physical distances.
Image projections alone do not uniquely determine absolute scale, while learned scale cues can be inaccurate in unfamiliar environments~\citep{wang2025vggt,wang2026moge}.
Sparse sensor observations provide a physical distance reference, while visual priors help reconstruct regions that sparse measurements alone cannot resolve.
Recent depth completion methods exploit this complementarity by combining foundation-model priors with sparse observations~\citep{lin2025prompting,wang2025depth,liu2024depthlab,viola2025marigold}.

However, applying this combination across sensors and environments requires accommodating variations in measurement density, spatial distribution, and noise caused by sensor configurations and processing pipelines.
Although training-based methods have improved robustness to such variation~\citep{zuo2025omni,bartolomei2024revisiting}, learned fusion can remain sensitive to observation conditions outside the training distribution~\citep{viola2025marigold,ke2026depth}.
Test-time optimization (TTO) reduces this dependence by using sparse observations as direct supervision to optimize latent representations, prompts, or model parameters at inference time~\citep{viola2025marigold,jeong2025test,ke2026depth}.
However, many existing methods optimize each image independently, repeating iterative adaptation across consecutive video frames and incurring substantial computational cost~\citep{viola2025marigold,jeong2025test}.
This limits their throughput and motivates reusing past computation and adaptation state to reduce the optimization required for each incoming frame.

\begin{figure*}[t!]
    \centering
    \includegraphics[width=1.0\linewidth]{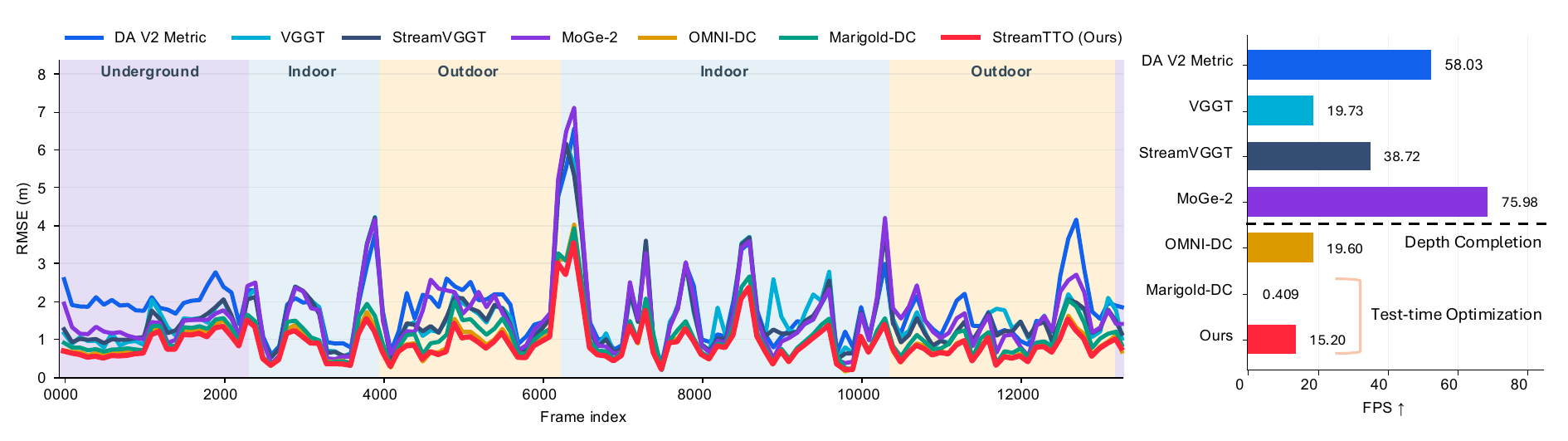}
    \vspace{-7mm}
    \caption{Depth accuracy and processing speed on MS-Depth TE06, with approximately 13,000 frames spanning indoor, outdoor, and underground environments. StreamTTO achieves the best accuracy among the evaluated methods at 15.2~FPS after initial adaptation.}
    \label{fig:teaser}
    \vspace{-7mm}
\end{figure*}

Figure~\ref{fig:teaser} illustrates the accuracy--speed trade-off among the evaluated methods, highlighting the computational burden of test-time optimization (TTO).
We propose StreamTTO to address this bottleneck through an empirically validated design principle: jointly reuse visual computation and adaptation progress across frames.
This reduces both per-update cost and the updates needed for online video depth completion.

A frozen causal feature extractor incorporates past visual context and computes features once per frame, preserving cache validity during decoder optimization.
The shared decoder and its optimizer state retain earlier depth corrections.
Causal Sliding-Window Optimization (CSWO) connects these forms of reuse, updating the decoder using cached features and sparse observations from current and recent frames.
The framework performs only two window passes per incoming frame after initialization.

To evaluate adaptation under sustained environmental changes, we introduce Multi-Scene Depth (MS-Depth): 120 sequences of approximately 20 minutes each, totaling approximately 40 hours of synchronized RGB--LiDAR recordings.
Test sequences span indoor, underground, and outdoor environments under daytime and nighttime conditions.
We evaluate StreamTTO on KITTI~\citep{geiger2013vision}, Bonn~\citep{palazzolo2019refusion}, NYUv2 Raw~\citep{silberman2012indoor}, TUM RGB-D~\citep{sturm2012benchmark}, DDAD~\citep{guizilini20203d}, and MS-Depth.
Ablations and same-backbone comparisons support the components' complementary roles.
StreamTTO achieves the lowest RMSE among the evaluated methods on KITTI, NYUv2 Raw, and DDAD, and averages approximately 15~FPS with VGGT and 25~FPS with MoGe-2~\cite{wang2026moge} on MS-Depth at $518 \times 392$ resolution.

\section{Related Work}

\paragraph{Training-based depth completion.}
Sparse-to-Dense~\citep{ma2018sparse}, CSPN~\citep{cheng2018depth}, GuideNet~\citep{tang2020learning}, NLSPN~\citep{park2020non}, and CompletionFormer~\citep{zhang2023completionformer} learn RGB--depth fusion or spatial propagation.
To improve cross-domain generalization, VPP4DC~\citep{bartolomei2024revisiting} reformulates completion as stereo matching, OMNI-DC~\citep{zuo2025omni} combines synthetic training with multi-resolution depth processing, and LDCM~\citep{yu2026large} predicts 3D points using a transformer.
PromptDA~\citep{lin2025prompting} and Any2Full~\citep{zhou2026any} condition Depth Anything on sparse measurements, while Prior Depth Anything~\citep{wang2025depth} incorporates incomplete metric priors.
With parameters fixed at inference, robustness to new sensors and scenes remains dependent on training-data coverage.

\paragraph{Test-time optimization for depth completion.}
TTO refines depth using sparse observations at inference.
\citet{park2024test} and ETA~\citep{chung2025eta} adapt
pretrained completion models online using source-trained
embedding and energy models, respectively, tying their
learned guidance to source observation distributions.
Foundation-model-based TTO uses target measurements without
such auxiliary guidance:
Marigold-DC~\citep{viola2025marigold} optimizes diffusion
latents, scale, and shift, while
TestPromptDC~\citep{jeong2025test} tunes visual prompts.
However, image-wise prompt optimization incurs substantial
computational overhead, as each update requires forward
and backward passes through the full network.
CAPA-LoRA~\citep{ke2026depth} reduces the number of trainable
parameters by updating only LoRA adapters in the image encoder.
Nevertheless, its default adaptation procedure uses
100 optimization steps, requiring repeated encoder feature
computation and backpropagation.
StreamTTO adapts causally as frames arrive, freezing the
feature extractor and optimizing only the decoder using
target sparse observations, without separately trained
auxiliary guidance.
Reusing cached features with causal temporal context,
recent observations in a sliding window, and the adapted
decoder reduces both the number of optimization steps
and the computational cost per incoming frame.
\paragraph{Continual depth completion.}
Continual depth adaptation includes selective module updates for throughput in MADNet~\citep{tonioni2019real}, analyses of forgetting on revisited domains~\citep{poggi2021continual}, and frame replay for stable online self-supervised monocular adaptation~\citep{kuznietsov2021comoda,vodisch2023codeps}.
For depth completion, UnCLe~\citep{chen2024uncle} benchmarks unsupervised continual learning across datasets, while ProtoDepth~\citep{rim2025protodepth} learns domain-specific prototypes over a frozen model to mitigate forgetting.
These settings use photometric and sparse-depth losses for sequential learning, emphasizing retention of earlier domains for subsequent multi-domain reuse.
StreamTTO continuously adapts during inference within a single sensor stream under a limited computational budget.
Its adapted decoder serves only the current stream, with no adaptation state transferred to independent sequences.
Preserving performance on earlier environments is not an explicit objective; the key criterion is depth accuracy in the current environment, including after environmental transitions, under a limited computational budget.

\begin{figure*}[t!]
 \centering
\includegraphics[width=1.0\linewidth]{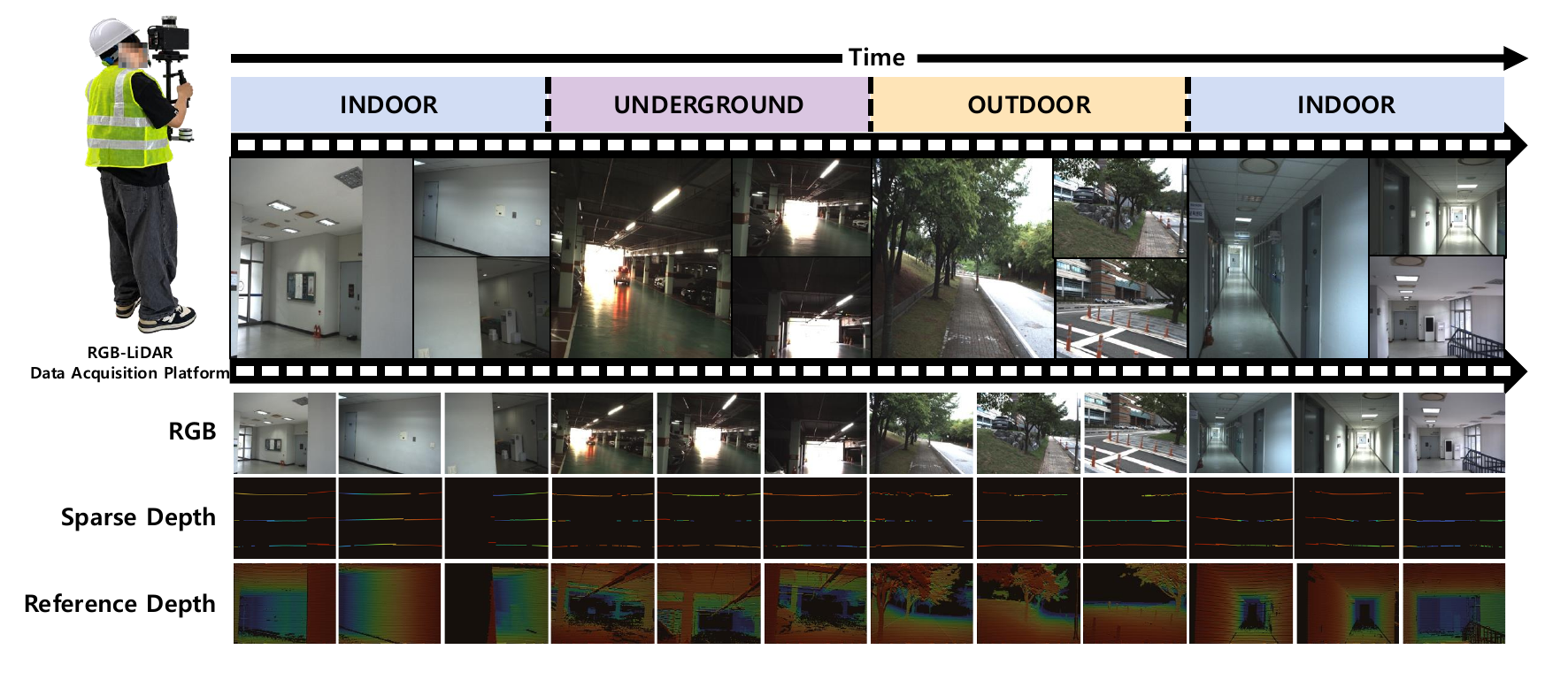} 
\vspace{-7mm}
\caption{An example MS-Depth test sequence repeatedly traversing indoor, underground, and outdoor environments before returning indoors. Sequences average approximately 20 minutes and are recorded continuously with a calibrated handheld camera--LiDAR rig. RGB images are paired with sparse and dense reference depth maps.}
 \label{fig:msdepth}
\end{figure*}

\section{MS-Depth Benchmark}
\label{sec:MS_Depth}

MS-Depth (Multi-Scene Depth) benchmarks video depth completion under continuous environmental changes.
It comprises 120 approximately 20-minute sequences totaling approximately 40 hours of synchronized RGB--LiDAR recordings.
The training and test splits contain 107 and 13 sequences, respectively; the latter covers eight locations/routes.

A rigid handheld platform carries an Ouster OS-1-128-SR 128-channel LiDAR operating at 10~Hz and three identical Point Grey Research Blackfly BFLY-PGE-31S4C cameras.
The central camera provides $2048 \times 1536$ RGB inputs; the other two form a stereo pair with a 20~cm baseline.
All sensors are hardware-synchronized by a trigger board.
Intrinsic and camera--LiDAR extrinsic calibration yield an image-plane calibration error of at most 0.9~pixels.
Recordings exclude moving pedestrians.
Reference depth is generated by projecting a single LiDAR scan into the central image, without multi-scan accumulation.
We triangulate RoMa~v2~\cite{edstedt2025roma} correspondences to obtain stereo depth and discard LiDAR samples with missing correspondences, left--right consistency errors exceeding 1~pixel, or depth differences exceeding $10\%$ of the LiDAR depth.
Sparse inputs retain eight LiDAR rings at fixed ring-index intervals.

Each test sequence continuously traverses indoor and outdoor spaces before returning to its starting point; nine also include underground segments (Figure~\ref{fig:msdepth}).
These recordings capture changes in illumination, scene structure, and depth range.
MS-Depth provides RGB images, sparse inputs, reference depth maps, IMU measurements, sequence-level day/night labels, and temporal annotations of indoor, outdoor, and underground segments.
StreamTTO uses the test split exclusively, with sparse observations for adaptation and reference depth for evaluation.
The training split supports future research.

\begin{figure*}[t!]
 \centering
\includegraphics[width=1.0\linewidth]{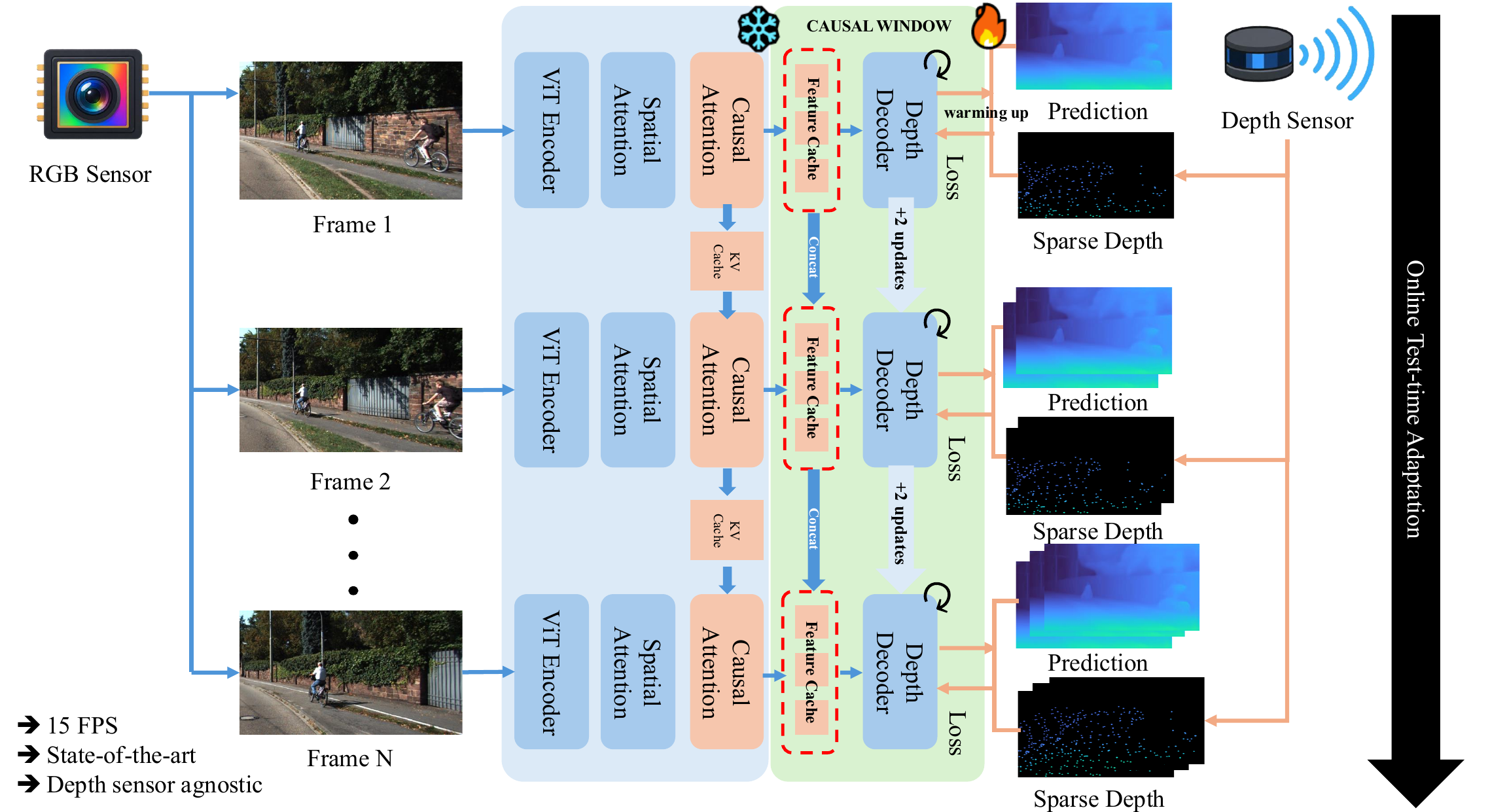} \caption{Overview of StreamTTO. A frozen ViT encoder with spatial and causal temporal attention extracts features using a key--value cache for past context. A separate feature cache stores the extracted features for batched optimization over a 3-frame causal window. Sparse depth observations supervise updates to the shared depth decoder, whose parameters are carried forward to subsequent frames. Only the decoder is optimized, avoiding repeated feature extraction.}
 \label{fig:method}
\end{figure*}

\section{StreamTTO}
\label{sec:streamtto}

\subsection{Overview}
\label{sec:overview}

We consider online video depth completion, where an RGB frame
$I_t$ and registered sparse metric depth $S_t$ arrive sequentially.
StreamTTO separates temporal representation from online adaptation:
a frozen spatiotemporal extractor supplies reusable visual features,
while a shared 2D depth decoder accumulates corrections from sparse
observations (Figure~\ref{fig:method}).
This decoder-only adaptation confines repeated forward and backward
computation to the depth decoder, leaving cached representations
unchanged by optimization.

The framework maintains two forms of memory.
\emph{Representation memory} consists of a temporal key--value
cache and a separate decoder-input feature cache.
Let $\mathcal{F}$ denote the frozen RGB encoder with spatial
and causal temporal attention.
Given the previous key--value memory $\mathcal{V}_{t-1}$,
it produces
\begin{equation}
    (Z_t,\mathcal{V}_t)
    =
    \mathcal{F}(I_t,\mathcal{V}_{t-1}),
    \label{eq:temporal_feature}
\end{equation}
where $Z_t$ encodes the current image and past visual context.
We cache $Z_t$ so that recent observations can be revisited
without repeating feature extraction.

\emph{Adaptation memory} consists of
$\Xi_t=(\theta_t,\mathcal{O}_t)$, where $\theta_t$ contains
the parameters of the shared depth decoder $\mathcal{D}$
and $\mathcal{O}_t$ is its optimizer state.
Causal Sliding-Window Optimization (CSWO) updates this state
using recent cached features stacked along the batch dimension.
Temporal reasoning remains in $\mathcal{F}$; the 2D decoder
requires no additional temporal operations.
The completed depth is
\begin{equation}
    \hat{D}_t
    =
    \operatorname{Align}_{S_t}
    [\mathcal{D}(Z_t;\theta_t)],
    \label{eq:depth_prediction}
\end{equation}
where $\operatorname{Align}_{S_t}$ estimates global scale
and shift from the sparse observations.

\subsection{Sparse-Depth-Guided Decoder Adaptation}
\label{sec:decoder_adaptation}

Let $D_t(\theta)=\mathcal{D}(Z_t;\theta)$ and
$\Omega_t=\operatorname{dom}(S_t)$ denote the decoder prediction
and valid sparse-observation locations, respectively.
For each prediction, scale and shift are estimated by
least-squares fitting:
\begin{equation}
(a_t^\star,b_t^\star)
=
\arg\min_{a,b}
\sum_{p\in\Omega_t}
\left(
a\,\operatorname{sg}[D_t(\theta)(p)]
+b-S_t(p)
\right)^2,
\label{eq:sparse_alignment}
\end{equation}
where $\operatorname{sg}$ denotes stop-gradient.
The frame-wise adaptation objective is
\begin{equation}
\mathcal{L}_t^{\mathrm{sp}}(\theta)
=
\frac{1}{|\Omega_t|}
\sum_{p\in\Omega_t}
\rho\!\left(
a_t^\star D_t(\theta)(p)+b_t^\star-S_t(p)
\right),
\qquad
\rho(e)=e^2+\lambda|e|,
\label{eq:sparse_adaptation_objective}
\end{equation}
where $\lambda$ weights the absolute residual penalty.
Alignment is recomputed before each loss evaluation, but
$a_t^\star$ and $b_t^\star$ are held constant during
backpropagation.
Gradients pass through the aligned prediction to all decoder
parameters, not through the alignment solver.
Thus, adaptation can correct spatially varying depth structure
rather than only global scale and shift.
We clip the global $L_2$ norm of decoder gradients before
each optimizer update.

Warm-starting from the previous decoder reuses learned corrections,
but fitting only the current observation set provides no explicit
constraint on their compatibility with recent frames.
CSWO retains these constraints in a causal buffer
$\mathcal{B}_t$ containing up to $W=3$ cached feature--depth
pairs $(Z_j,S_j)$.
The buffer grows from the first frame and then replaces its
oldest entry at each arrival.

Starting from pretrained decoder parameters and a fresh optimizer
state, we perform $K_{\mathrm{boot}}=20$ bootstrap updates
on the first frame.
Subsequent frames inherit the adapted decoder and optimizer
state and perform $R=2$ passes over the available window.
Each pass processes the entire window as one batch and applies
one optimizer update.
The batch loss is the normalized weighted average of
$\mathcal{L}_j^{\mathrm{sp}}$, assigning weight 2 to the current
frame and 1 to each past frame.

Consecutive full windows share $W-1$ observations, so the inherited
decoder has already been adapted to most of the next objective.
Revisiting these constraints supports retaining corrections across
frames while incorporating new measurements.
Unlike optimizing an available sequence jointly, CSWO uses only
current and past observations, with bounded memory and two
decoder updates per subsequent frame.

StreamTTO is designed to reduce both the cost of each optimization update
and the number of updates required per incoming frame.
Appendix~\ref{sec:cached_adaptation_analysis} formalizes the consistency
of cached causal representations and bounds the change in the decoder's
backpropagated gradient across overlapping windows, providing a
conditional rationale for warm-start adaptation with few updates.
\begin{table*}[t]
\centering
\setlength{\tabcolsep}{3.0pt}
\renewcommand{\arraystretch}{1.08}
\setlength{\heavyrulewidth}{1.2pt}
\setlength{\lightrulewidth}{0.55pt}
\setlength{\cmidrulewidth}{0.45pt}

\resizebox{0.8\textwidth}{!}{%
\begin{tabular}{@{}l*{5}{cc}c@{}}
\toprule

\multirow{2}{*}{\textbf{Method}}
& \multicolumn{2}{c}{\textbf{KITTI}}
& \multicolumn{2}{c}{\textbf{Bonn}}
& \multicolumn{2}{c}{\textbf{NYUv2 Raw}}
& \multicolumn{2}{c}{\textbf{TUM RGB-D}}
& \multicolumn{2}{c}{\textbf{DDAD}}
& \textbf{Speed} \\

\cmidrule(lr){2-3}
\cmidrule(lr){4-5}
\cmidrule(lr){6-7}
\cmidrule(lr){8-9}
\cmidrule(lr){10-11}
\cmidrule(lr){12-12}

& \textbf{RMSE}$\downarrow$ & \textbf{MAE}$\downarrow$
& \textbf{RMSE}$\downarrow$ & \textbf{MAE}$\downarrow$
& \textbf{RMSE}$\downarrow$ & \textbf{MAE}$\downarrow$
& \textbf{RMSE}$\downarrow$ & \textbf{MAE}$\downarrow$
& \textbf{RMSE}$\downarrow$ & \textbf{MAE}$\downarrow$
& \textbf{Avg. FPS}$\uparrow$ \\

\midrule
\multicolumn{12}{@{}l}{\textbf{Depth Foundation Models}} \\
\addlinespace[2pt]

Depth Anything V2 Metric
& 4.1902 & 2.1348 & 0.2278 & 0.1380
& 0.1660 & 0.1060 & 0.2705 & 0.1410
& 15.7768 & 8.7141 & 26.68 \\

UniDepthV2
& 2.9397 & 1.2993 & 0.1832 & 0.0811
& 0.1327 & 0.0698 & 0.2674 & 0.1265
& 13.2580 & 5.7289 & \textbf{37.47} \\

Video Depth Anything Metric (online)
& 4.3396 & 2.3805 & 0.1992 & 0.1125
& 0.2513 & 0.1380 & 0.2793 & 0.1490
& 15.4656 & 7.8687 & 35.53 \\

\midrule
\multicolumn{12}{@{}l}{\textbf{3D Foundation Models}} \\
\addlinespace[2pt]

VGGT
& 2.8888 & 1.2200 & 0.1856 & 0.0860
& 0.1861 & 0.0886 & 0.2191 & 0.1024
& 12.4564 & 6.1351 & 9.46 \\

MoGe-2
& 2.8805 & 1.2522 & 0.1645 & 0.0775
& 0.1222 & 0.0707 & 0.2320 & 0.1083
& 10.4588 & 4.3600 & \underline{35.88} \\

StreamVGGT
& 2.9495 & 1.3631 & 0.1655 & 0.0781
& 0.1587 & 0.0794 & 0.2051 & 0.0962
& 12.7959 & 6.5728 & 17.47 \\

ABot-Recon~\citep{han2026revisiting}
& 2.5400 & 1.0041 & 0.2083 & 0.0935
& 0.2276 & 0.1174 & 0.1844 & 0.0915
& 9.3542 & 3.3378 & 20.27 \\

\midrule
\multicolumn{12}{@{}l}{\textbf{Depth Completion Models}} \\
\addlinespace[2pt]


PriorDA v1.1
& 1.8406 & 0.6121 & 0.0845 & 0.0268
& 0.0730 & 0.0306 & 0.0920 & 0.0242
& 8.4179 & 2.4995 & 23.75 \\

OMNI-DC v1.1
& 1.3557 & 0.4219
& 0.0700 & \textbf{0.0146}
& 0.0526 & \textbf{0.0147}
& 0.0563 & \textbf{0.0113}
& 7.6406 & 1.7980
& 9.07 \\

Marigold-DC
& 1.7106 & 0.6890 & 0.0723 & 0.0249
& 0.0731 & 0.0315 & 0.0837 & 0.0273
& 9.9222 & 3.8292 & 0.221 \\

TestPromptDC
& 1.1266 & 0.3441
& \underline{0.0514} & 0.0190
& 0.0495 & 0.0229
& \textbf{0.0514} & 0.0163
& 8.2693 & 2.2167
& 0.256 \\

CAPA-LoRA (offline)$^{\dagger}$
& 1.0902 & \textbf{0.2684}
& 0.0566 & 0.0179
& 0.0546 & 0.0219
& \underline{0.0560} & \underline{0.0144}
& 7.6213 & \underline{1.7592}
& 0.780 \\

\midrule

VGGT + Causal-Attn &
3.0817 & 1.3334 & 0.1723 & 0.0792 & 0.1602 & 0.0832 & 0.2150 & 0.0988 & 12.7059 & 6.6541 &17.52 \\

\rowcolor{gray!15}VGGT + StreamTTO
& \underline{1.0223} & 0.3920
& 0.0592 & 0.0210
& \underline{0.0403} & 0.0220
& 0.0738 & 0.0274
& \underline{7.1448} & 2.3483
& 7.14 \\

MoGe-2 + Causal-Attn&
2.8130 & 1.2086 & 0.1641 & 0.0793 & 0.1194 & 0.0732 & 0.2414 & 0.1075 & 10.2409 & 4.1452
& 33.28 \\

\rowcolor{gray!15}MoGe-2 + StreamTTO
& \textbf{0.7796} & \underline{0.2763}
& \textbf{0.0430} & \underline{0.0155}
& \textbf{0.0316} & \underline{0.0173}
& 0.0586 & 0.0195
& \textbf{5.2864} & \textbf{1.7124}
& 11.99 \\

\bottomrule
\end{tabular}%
}

\caption{
Quantitative comparison of depth foundation models,
3D foundation models, and depth completion methods
on five public benchmarks.
RMSE and MAE are reported in meters, excluding sparse-input
observation locations from accuracy evaluation.
Bold and underlined values indicate the best and second-best
results across all methods, respectively, including ties.}
\label{tab:main_depth_results_with_fps}
\end{table*}

\section{Experiments}

\subsection{Experimental Setup}
\label{sec:experimental-setup}
\paragraph{Datasets.}
We evaluate on 110-frame clips sampled from eight KITTI
($1216 \times 352$), five Bonn RGB-D Dynamic, eight NYUv2 Raw
($640 \times 480$), and eight TUM RGB-D ($640 \times 480$)
sequences, totaling 880, 550, 880, and 880 frames, respectively.
For DDAD ($1936 \times 1216$), we use all 3,950 front-camera
frames from the 50 scenes in its official validation split.
Details of public-benchmark sequence selection and clip
construction are provided in
Appendix~\ref{app:evaluation_subset_construction}.
For MS-Depth, we reserve TE04 as a validation sequence for
method development, hyperparameter selection, and ablation studies.
We evaluate on the remaining 12 test sequences using the
hyperparameters selected on TE04, without further tuning,
and report the results in Table~\ref{tab:ms_depth_munji3_results}.
These 12 sequences are excluded from method development
and hyperparameter selection.
MS-Depth RGB images are resized to $518 \times 392$ for model input.
Further details of MS-Depth are provided in
Appendix~\ref{app:dataset-specifications}.

\paragraph{Sparse-depth inputs.}
For KITTI, DDAD, and MS-Depth, we project measurements from at least 8 LiDAR rings onto the RGB image plane, keeping the ring selection fixed per sequence.
For Bonn, NYUv2 Raw, and TUM RGB--D, we uniformly sample 10\% of valid reference-depth pixels without replacement.
NYUv2 Raw results at 3\% and 1\% density are provided in Appendix~\ref{sec:observation_density}.
All methods requiring sparse depth receive identical observations;
evaluation uses only the remaining valid reference-depth pixels.

\paragraph{Implementation details.}
We use official code and checkpoints, including
VGGT~\citep{wang2025vggt} and MoGe-2~\citep{wang2026moge}.
CAPA-LoRA is evaluated offline using its official
optimization settings.
All other methods operate causally.
Additional scale--shift alignment is applied only to depth
and 3D foundation-model outputs.
StreamTTO freezes the feature extractor and updates the full
shared decoder using a 6-frame causal KV context and a
3-frame CSWO window.
After $K_{\mathrm{boot}}=20$ first-frame updates, it performs
$R=2$ window passes per subsequent frame, retaining decoder
and optimizer states.
We use AdamW with learning rate $3\times10^{-5}$,
$\lambda=0.1$, and gradient-norm clipping at $1.0$.
We select $W$, $R$, and the learning rate on TE04 and keep
them fixed across all evaluation datasets.
Although the MoGe-2 variant achieves higher accuracy and
throughput in our main experiments, all ablations use VGGT
with causal training following
StreamVGGT~\citep{zhuo2026streaming}.
We choose VGGT to assess the proposed adaptation framework
on a general-purpose 3D foundation model, beyond backbones
specialized for monocular metric-depth estimation.
All experiments run on NVIDIA B200 GPUs.
Further implementation and ablation details are provided in
Appendix~\ref{app:adaptation_ablation_details}.

\subsection{Quantitative Results}

\paragraph{Public Benchmarks.}
Table~\ref{tab:main_depth_results_with_fps} shows that most
depth completion methods outperform foundation models calibrated
with the same sparse observations.
MoGe-2-based StreamTTO achieves the lowest RMSE on KITTI,
Bonn, NYUv2 Raw, and DDAD.
TestPromptDC leads in TUM RGB-D RMSE, while OMNI-DC achieves
the lowest MAE on the three indoor benchmarks.
MoGe-2 improves both RMSE and MAE over the VGGT variant
on all five benchmarks, potentially benefiting from its
pretraining for metric geometry estimation~\citep{wang2026moge}.

The VGGT and MoGe-2 variants achieve average throughputs
of 7.14 and 11.99~FPS, respectively.
Both use two window passes per frame after initialization,
allowing the lighter MoGe-2 backbone to increase throughput
with the same number of optimizer updates.
Despite performing backpropagation and parameter updates,
the MoGe-2 variant also outpaces training-based OMNI-DC
(9.07~FPS), which keeps its parameters fixed at inference.
It achieves approximately 54, 47, and 15 times the throughput
of Marigold-DC, TestPromptDC, and offline CAPA-LoRA, respectively.
Further results with VGGT and MoGe-2, including comparisons
of adaptation strategies on the same backbone, are provided
in Appendix~\ref{sec:shared_backbone_adaptation}.

\begin{table*}[t]
\centering
\setlength{\tabcolsep}{3.2pt}
\renewcommand{\arraystretch}{1.10}

\resizebox{0.7\textwidth}{!}{%
\begin{tabular}{@{}l*{4}{cc}c@{}}
\toprule[1.2pt]

\multicolumn{10}{c}{\textbf{MS-Depth}} \\
\midrule

\multirow{3}{*}{\textbf{Method}}
& \multicolumn{4}{c}{\textbf{Indoor}}
& \multicolumn{4}{c}{\textbf{Outdoor}}
& \multirow{3}{*}{\textbf{FPS}$\uparrow$} \\

\cmidrule(lr){2-5}
\cmidrule(lr){6-9}

& \multicolumn{2}{c}{\textbf{Underground}}
& \multicolumn{2}{c}{\textbf{Aboveground}}
& \multicolumn{2}{c}{\textbf{Day}}
& \multicolumn{2}{c}{\textbf{Night}}
& \\

\cmidrule(lr){2-3}
\cmidrule(lr){4-5}
\cmidrule(lr){6-7}
\cmidrule(lr){8-9}

& \textbf{RMSE}$\downarrow$ & \textbf{MAE}$\downarrow$
& \textbf{RMSE}$\downarrow$ & \textbf{MAE}$\downarrow$
& \textbf{RMSE}$\downarrow$ & \textbf{MAE}$\downarrow$
& \textbf{RMSE}$\downarrow$ & \textbf{MAE}$\downarrow$
& \\

\midrule
\multicolumn{10}{@{}l}{\textbf{Depth Foundation Models}} \\
\addlinespace[2pt]

Depth Anything V2 Metric
& 2.3864 & 1.7842 & 1.7433 & 0.8399
& 2.7577 & 1.6428 & 2.4214 & 1.3059
& 59.52 \\

UniDepthV2
& 1.4090 & 0.8418 & 1.6569 & 0.5920
& 3.6521 & 2.3448 & 2.9081 & 1.6813
& \textbf{83.16} \\

Video Depth Anything Metric (online)
& 1.6474 & 1.0254 & 1.6382 & 0.6428
& 3.8411 & 2.6001 & 3.0138 & 1.8128
& \underline{78.29} \\

\midrule
\multicolumn{10}{@{}l}{\textbf{3D Foundation Models}} \\
\addlinespace[2pt]

VGGT
& 1.5202 & 0.8582 & 1.7391 & 0.6478
& 2.2774 & 1.2286 & 2.5653 & 1.4303
& 20.39 \\

MoGe-2
& 1.3088 & 0.7583 & 1.7428 & 0.6427
& 2.5355 & 1.4393 & 2.4911 & 1.3567
& 78.22 \\

StreamVGGT
& 1.5370 & 0.9679 & 1.6938 & 0.6296
& 2.2854 & 1.2491 & 2.4262 & 1.3171
& 39.81 \\

ABot-Recon~\citep{han2026revisiting}
& 1.3468 & 0.7188 & 1.7619 & 0.6690
& 2.3350 & 1.2698 & 2.2844 & 1.1888
& 47.27 \\

\midrule
\multicolumn{10}{@{}l}{\textbf{Depth Completion Models}} \\
\addlinespace[2pt]


PriorDA v1.1
& 0.9434 & 0.3200 & 1.1660 & 0.2491
& 1.5273 & 0.5098 & 1.3949 & 0.4290
& 54.42 \\

OMNI-DC v1.1
& 0.8746 & \textbf{0.2046}
& 1.1042 & \underline{0.1815}
& 1.4891 & \underline{0.3515}
& 1.3572 & \underline{0.2927}
& 17.84 \\

Marigold-DC
& 0.9806 & 0.3386 & 1.1964 & 0.2999
& 1.8212 & 0.6611 & 1.5550 & 0.5359
& 0.426 \\

TestPromptDC
& 0.9487 & 0.2958 & 1.1235 & 0.2346
& 1.5795 & 0.4385 & 1.4509 & 0.4366
& 0.543 \\

CAPA-LoRA (offline)
& 0.9062 & 0.2441
& 1.2494 & 0.2175
& 1.4502 & 0.3590
& 1.3923 & 0.3300
& 1.652 \\

\midrule

VGGT + Causal-Attn
& 1.5222 & 0.9521
& 1.8435 & 0.6425
& 2.3414 & 1.4592
& 2.3252 & 1.3433
& 39.84 \\

\rowcolor{gray!15}VGGT + StreamTTO
& \underline{0.8077} & 0.2610
& \underline{1.0728} & 0.2211
& \underline{1.3666} & 0.3668
& \underline{1.2603} & 0.3497
& 15.73 \\

MoGe-2 + Causal-Attn
& 1.5088 & 0.7783
& 1.9442 & 0.7325
& 2.7352 & 1.5499
& 2.7917 & 1.6565
& 70.25 \\

\rowcolor{gray!15}MoGe-2 + StreamTTO
& \textbf{0.5899} & \underline{0.2063}
& \textbf{0.8339} & \textbf{0.1647}
& \textbf{0.9950} & \textbf{0.2779}
& \textbf{0.8846} & \textbf{0.2558}
& 25.90 \\

\bottomrule[1.2pt]
\end{tabular}%
}
\caption{Results on the 12 MS-Depth test sequences excluding TE04.
RMSE and MAE (m) exclude shared sparse-input locations.
Depth and 3D foundation models use sparse-depth-based scale--shift
alignment; depth completion methods receive no additional alignment.}
\label{tab:ms_depth_munji3_results}
\end{table*}

\paragraph{MS-Depth Dataset.}
MS-Depth reveals a different performance trend from the
public benchmarks under continuous environmental changes
(Table~\ref{tab:ms_depth_munji3_results}).
TestPromptDC outperforms OMNI-DC in RMSE on four of the
five public benchmarks, but this ranking reverses across
all four MS-Depth conditions.
StreamTTO maintains its accuracy advantage, with the MoGe-2
and VGGT variants achieving the lowest and second-lowest
RMSE, respectively, in all four conditions.
With only two window passes per incoming frame after
initialization, the VGGT and MoGe-2 variants achieve
throughputs of 15.73 and 25.90~FPS, respectively.
Taken together with the public-benchmark results, these
findings provide evidence that sparse-depth-guided test-time
optimization can help depth foundation models adapt to
diverse sensing conditions and changing environments.
Further comparisons of adaptation strategies using the
same VGGT and MoGe-2 backbones are provided in
Appendix~\ref{sec:shared_backbone_adaptation}.

\begin{figure*}[t!]
 \centering
\includegraphics[width=0.8\linewidth]{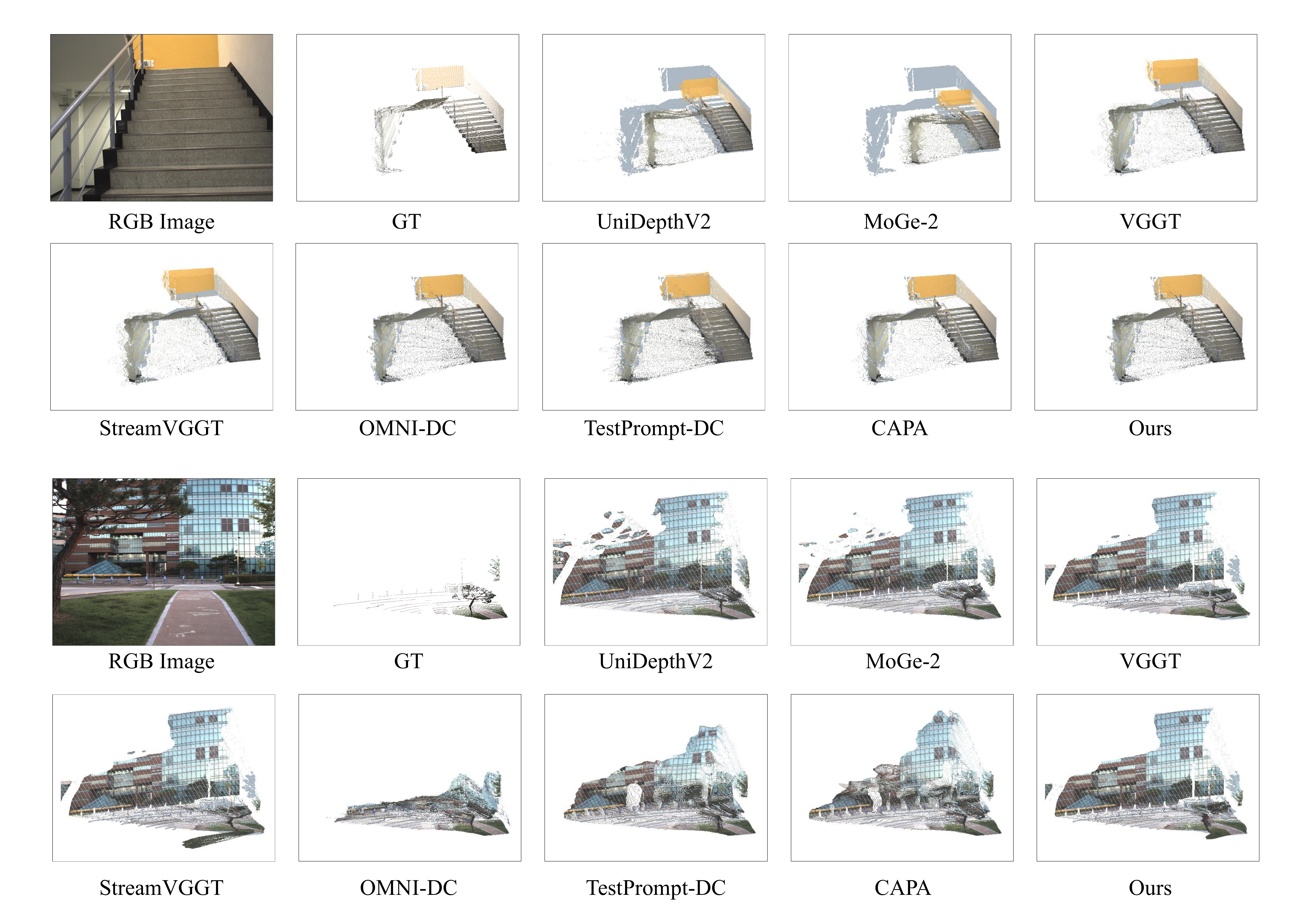} 

\caption{Qualitative 3D comparisons for an indoor staircase and an outdoor building. Predicted depth is back-projected into 3D and colored using the corresponding RGB values. Reference geometry is shown in gray for comparison of scale and structure.}
 \label{fig:vis_main}
\end{figure*}

\subsection{Qualitative Results}

Figure~\ref{fig:vis_main} compares baselines with VGGT-based StreamTTO.
Predicted depth is back-projected into RGB-colored point clouds, with ground-truth points overlaid in gray.
In the indoor staircase example, VGGT and StreamVGGT exhibit visible metric-scale discrepancies relative to the reference.
UniDepthV2, MoGe-2, and TestPromptDC also show geometric misalignment, whereas OMNI-DC, CAPA-LoRA, and StreamTTO more closely reproduce the staircase geometry.

Outdoors, OMNI-DC exhibits severe vertical compression, while TestPromptDC and CAPA-LoRA distort the building height and facade.
VGGT and StreamVGGT better preserve the overall building structure in this example.
In these examples, StreamTTO is the only compared method that achieves close indoor metric alignment while preserving the outdoor building geometry.

\subsection{Ablation Studies}

We conduct the component ablations in Table~\ref{tab:ablations}
with VGGT on the MS-Depth TE04 development sequence and report
mean RMSE over 10 seeds.
Unless otherwise specified, we use a 6-frame causal KV context,
full-decoder adaptation, warm-start with two passes per frame,
a CSWO window of $W=3$, and 20 bootstrap updates.

Table~\ref{tab:abl_context} shows that warm-start adaptation
benefits more from temporal context than from single-frame processing.
With causal KV context, warm-start reduces RMSE from 1.3725
to 0.8694~m, reaching accuracy comparable to recomputed context
(0.8718~m) at approximately one third of the FLOPs.

Tables~\ref{tab:abl_target} and~\ref{tab:abl_lora} compare
adaptation targets and parameterizations, including
encoder-LoRA as used in CAPA~\citep{ke2026depth}.
Full-decoder adaptation achieves lower RMSE than encoder-only
updates or joint image-encoder and decoder updates at
substantially lower computational cost.
Decoder-LoRA slightly reduces computation but increases RMSE
from 0.8694 to 0.9205~m, whereas full-LoRA requires
$20.8\times$ as many FLOPs for similar accuracy.

Table~\ref{tab:abl_init} shows that additional updates
have different effects under reset and warm-start adaptation.
Increasing passes from 2 to 20 improves reset RMSE
from 1.3725 to 0.8753~m, but worsens warm-start RMSE
from 0.8694 to 0.9846~m.
Warm-start with two passes therefore achieves accuracy
comparable to reset with 20 passes at substantially lower cost.
Appendix~\ref{sec:warmstart_baselines} further evaluates
warm-start adaptation in other TTO methods.
Across the five public benchmarks, warm-starting reduces RMSE
by an average of 2.50\% for TestPromptDC and 5.72\% for CAPA-LoRA
relative to their original configurations.
The complete VGGT-based StreamTTO framework achieves an average
RMSE reduction of 63.35\% relative to its causal backbone
without test-time parameter updates.
These averages are computed from the relative RMSE reduction
on each benchmark.

Table~\ref{tab:abl_window} evaluates the CSWO window size.
Increasing $W$ from 1 to 3 reduces RMSE from 0.9314 to 0.8694~m.
Expanding to $W=6$ provides a further 1.6\% reduction
but increases FLOPs by 57.8\%, while $W=12$ worsens RMSE
to 0.9544~m.
We therefore use $W=3$ to balance accuracy and computation.

Tables~\ref{tab:abl_kv} and~\ref{tab:abl_boot} show that
extending KV context beyond 6 frames provides negligible
accuracy gains, while 20--100 bootstrap updates yield
similar sequence-level RMSE.

\paragraph{Adaptation across environmental transitions.}
\label{sec:transition_analysis}

We compare W1 and W3 under matched decoder training FLOPs
on 51 predefined transition clips from the 12 MS-Depth
evaluation sequences, excluding TE04.
Each clip is initialized approximately 4 seconds before
the transition and processed until 10 seconds after it,
retaining decoder parameters and optimizer state
without resetting at $t=0$.
RMSE is evaluated only over the post-transition interval
$[0,10)$ seconds.
Both configurations share the same frozen backbone
and 6-frame causal KV context.
W1 performs 6 current-frame updates, whereas W3 performs
2 passes over a 3-frame window, matching decoder training
FLOPs at 4.063~T per incoming frame once the window is full.

Table~\ref{tab:transition_matched_flops} shows that
W3 lowers RMSE across all six transition types,
reducing overall RMSE from 1.2146 to 1.1340~m (6.6\%)
while increasing throughput from 10.5 to 15.7~FPS.
These results support the benefit of reusing recent observations
for post-transition depth accuracy under the same decoder
training budget.
\begin{table}[t]
\centering
\small
\setlength{\tabcolsep}{3pt}
\renewcommand{\arraystretch}{1.08}
\captionsetup[subtable]{
    font=small,
    justification=raggedright,
    singlelinecheck=false,
    skip=4pt
}

\begin{subtable}[b]{0.32\linewidth}
\centering
\begin{adjustbox}{max width=\linewidth}
\begin{tabular}{@{}lcccc@{}}
\toprule
& \multicolumn{3}{c}{RMSE$\downarrow$ (m)}
& \shortstack{FLOPs$\downarrow$\\(T/frame)} \\
\cmidrule(lr){2-4}
Context & Scale--shift & Reset & Warm-start
& Warm-start \\
\midrule
Single frame
& 1.4213
& $1.2778 \pm 0.0352$
& $1.1333 \pm 0.0514$
& 6.496 \\
6 frames, recomputed
& 1.2883
& $1.1591 \pm 0.0280$
& $0.8718 \pm 0.0401$
& 20.727 \\
\rowcolor{gray!20}
6 frames, causal KV
& 1.5063 
& $1.3725 \pm 0.0189$
& $0.8694 \pm 0.0532$
& 7.028 \\
\bottomrule
\end{tabular}
\end{adjustbox}
\caption{\textbf{Temporal context.}}
\label{tab:abl_context}
\end{subtable}\hfill
\begin{subtable}[b]{0.32\linewidth}
\centering
\begin{adjustbox}{max width=\linewidth}
\begin{tabular}{@{}lccc@{}}
\toprule
target & \#params & RMSE$\downarrow$ & FLOPs$\downarrow$ \\
\midrule
encoder & 304.372M & $0.9306 \pm 0.0287$ & 149.538 \\
\rowcolor{gray!20}
decoder & 32.655M & $0.8694 \pm 0.0532$ & 7.028 \\
full & 337.026M & $0.9271 \pm 0.0492$ & 150.892 \\
\bottomrule
\end{tabular}
\end{adjustbox}
\caption{\textbf{Adaptation target.}}
\label{tab:abl_target}
\end{subtable}\hfill
\begin{subtable}[b]{0.32\linewidth}
\centering
\begin{adjustbox}{max width=\linewidth}
\begin{tabular}{@{}lccc@{}}
\toprule
target & \#params & RMSE$\downarrow$ & FLOPs$\downarrow$ \\
\midrule
encoder-LoRA & 0.393M & $0.8884 \pm 0.0157$ & 145.768 \\
decoder-LoRA & 0.357M & $0.9205 \pm 0.0536$ & 5.719 \\
full-LoRA & 0.750M & $0.8706 \pm 0.0271$ & 145.885 \\
\bottomrule
\end{tabular}
\end{adjustbox}
\caption{\textbf{Low-rank adaptation.}}
\label{tab:abl_lora}
\end{subtable}

\par\vspace{8pt}

\begin{subtable}[b]{0.235\linewidth}
\centering
\begin{adjustbox}{max width=\linewidth}
\begin{tabular}{@{}lccc@{}}
\toprule
init & passes & RMSE$\downarrow$ & FLOPs$\downarrow$ \\
\midrule
reset & 2 & $1.3725 \pm 0.0189$ & 7.028 \\
reset & 20 & $0.8753 \pm 0.0526$ & 43.595 \\
\rowcolor{gray!20}
warm-start & 2 & $0.8694 \pm 0.0532$ & 7.028 \\
warm-start & 20 & $0.9846 \pm 0.0232$ & 43.595 \\
\bottomrule
\end{tabular}
\end{adjustbox}
\caption{\textbf{Initialization.}}
\label{tab:abl_init}
\end{subtable}\hfill
\begin{subtable}[b]{0.235\linewidth}
\centering
\begin{adjustbox}{max width=\linewidth}
\begin{tabular}{@{}ccc@{}}
\toprule
$W$ & RMSE$\downarrow$ & FLOPs$\downarrow$ \\
\midrule
1 & $0.9314 \pm 0.0102$ & 4.320 \\
\rowcolor{gray!20}
3 & $0.8694 \pm 0.0532$ & 7.028 \\
6 & $0.8551 \pm 0.0272$ & 11.091 \\
12 & $0.9544 \pm 0.0548$ & 19.217 \\
\bottomrule
\end{tabular}
\end{adjustbox}
\caption{\textbf{CSWO window.}}
\label{tab:abl_window}
\end{subtable}\hfill
\begin{subtable}[b]{0.235\linewidth}
\centering
\begin{adjustbox}{max width=\linewidth}
\begin{tabular}{@{}ccc@{}}
\toprule
KV context & RMSE$\downarrow$ & FLOPs$\downarrow$ \\
\midrule
1 & $1.1333 \pm 0.0514$ & 6.496 \\
3 & $0.8918 \pm 0.0308$ & 6.709 \\
\rowcolor{gray!20}
6 & $0.8694 \pm 0.0532$ & 7.028 \\
12 & $0.8697 \pm 0.0540$ & 7.668 \\
\bottomrule
\end{tabular}
\end{adjustbox}
\caption{\textbf{KV context length.}}
\label{tab:abl_kv}
\end{subtable}\hfill
\begin{subtable}[b]{0.235\linewidth}
\centering
\begin{adjustbox}{max width=\linewidth}
\begin{tabular}{@{}cc@{}}
\toprule
$K_{\mathrm{boot}}$ & RMSE$\downarrow$ \\
\midrule
\rowcolor{gray!20}
20 & $0.8694 \pm 0.0532$ \\
40 & $0.8727 \pm 0.0572$ \\
60 & $0.8749 \pm 0.0363$ \\
100 & $0.8684 \pm 0.0130$ \\
\bottomrule
\end{tabular}
\end{adjustbox}
\caption{\textbf{Bootstrap updates.}}
\label{tab:abl_boot}
\end{subtable}

\caption{StreamTTO ablations on MS-Depth TE04 with the VGGT backbone.
RMSE (m) is reported as mean $\pm$ standard deviation over 10 seeds.
In the adaptation-target ablation, ``full'' denotes joint updates
to the image encoder and depth decoder, while the spatial and
causal temporal attention modules remain frozen.
FLOPs (T/frame) include feature extraction and test-time adaptation
for each configuration at an input resolution of $518 \times 392$.
Default settings are highlighted in gray:
6-frame causal KV context, full-decoder adaptation,
warm-start with 2 passes, $W=3$, and 20 bootstrap updates.}
\label{tab:ablations}
\end{table}

\begin{table*}[t]
\centering
\setlength{\tabcolsep}{5pt}
\renewcommand{\arraystretch}{1.12}
\resizebox{\textwidth}{!}{%
\begin{tabular}{@{}l*{9}{r}@{}}
\toprule
Method
& \shortstack{Indoor $\rightarrow$\\Outdoor}
& \shortstack{Outdoor $\rightarrow$\\Indoor}
& \shortstack{Indoor $\rightarrow$\\Underground}
& \shortstack{Underground $\rightarrow$\\Indoor}
& \shortstack{Outdoor $\rightarrow$\\Underground}
& \shortstack{Underground $\rightarrow$\\Outdoor}
& Overall
& FPS$\uparrow$
& \shortstack{Training FLOPs$\downarrow$\\(T/frame)} \\
\midrule
W1 (6 updates)
& $1.0363 \pm 0.0600$
& $1.5833 \pm 0.0469$
& $1.1038 \pm 0.0116$
& $0.9515 \pm 0.0435$
& $0.8663 \pm 0.0338$
& $0.7513 \pm 0.0368$
& $1.2146 \pm 0.0360$
& 10.5
& \textbf{4.063} \\

W3 (2 passes)
& $\mathbf{0.9762 \pm 0.0241}$
& $\mathbf{1.4970 \pm 0.0386}$
& $\mathbf{1.0047 \pm 0.0268}$
& $\mathbf{0.8789 \pm 0.0376}$
& $\mathbf{0.7470 \pm 0.0424}$
& $\mathbf{0.6459 \pm 0.0406}$
& $\mathbf{1.1340 \pm 0.0286}$
& \textbf{15.7}
& \textbf{4.063} \\
\bottomrule
\end{tabular}%
}
\caption{Post-transition RMSE (m, $\downarrow$) over $[0,10)$~s
on 51 MS-Depth transition clips, excluding TE04.
Results are reported as mean $\pm$ standard deviation
across 10 random seeds.
W1 and W3 use matched decoder training FLOPs,
which exclude frozen feature extraction.}
\label{tab:transition_matched_flops}
\vspace{-5mm}
\end{table*}

\section{Conclusion}
We presented StreamTTO, an online framework for video depth completion that exploits temporal redundancy to reuse both visual computation and depth corrections.
By caching temporally contextualized features and updating a shared depth decoder through Causal Sliding-Window Optimization, StreamTTO replaces repeated image-wise optimization with brief, cumulative updates across a video.
Experiments on five public benchmarks and MS-Depth demonstrate competitive depth accuracy, while our ablations show that combining temporal memory with window optimization can outperform independent 20-step TTO at substantially lower computational cost.
We also introduced MS-Depth, a 40-hour dataset spanning diverse environments and continuous scene transitions, to support the study
of long-horizon adaptation.
Our results highlight that efficient video depth completion depends not only on reducing optimization steps, but also on making previous
computation and corrections reusable across frames.
Future work will focus on improving processing speed and adaptation robustness under prolonged environmental changes. Code and data will be released.

\section*{Ethics Statement}
MS-Depth was collected by the authors in public places, with assistance from participants who were compensated for their time and effort. Faces and vehicle license plates were blurred to reduce privacy risks. The data collection and preprocessing procedures are described in Section~\ref{sec:MS_Depth} and Appendix~\ref{app:dataset-specifications}.

\section*{Reproducibility Statement}
The proposed method and experimental setup are described in Section~\ref{sec:streamtto} and Section~\ref{sec:experimental-setup}, respectively. Appendix~\ref{app:adaptation_ablation_details} provides additional implementation details, including offline fine-tuning, online adaptation, and runtime measurement. Dataset specifications and evaluation subset construction are documented in Appendices~\ref{app:dataset-specifications} and~\ref{app:evaluation_subset_construction}. The assumptions and derivations underlying our analysis of cached online decoder adaptation are provided in Appendix~\ref{sec:cached_adaptation_analysis}.
The supplementary materials include some selected RGB–LiDAR overlay visualizations from MS-Depth TE04 and JSON annotations for all 13 test sequences describing day/night conditions and environment transitions within each sequence, as illustrated in Figure~\ref{fig:msdepth-environment-timeline}.
We will publicly release the source code and the full MS-Depth dataset upon acceptance.

\section*{AI Use Statement}
We used AI tools to assist with language editing and translation, code cleanup, and dataset organization and cleaning. The core research ideas, methodology, and implementation were developed by the authors, who take full responsibility for the final manuscript, code, and dataset.

\bibliography{iclr2027_conference}
\bibliographystyle{iclr2027_conference}

\newpage
\appendix
\section{Appendix}

\subsection{Broader implications}

Test-time optimization (TTO) allows a deployed depth completion system
to adapt using sparse sensor measurements without revisiting its
original training data. A central practical challenge is fitting
iterative adaptation within the computational budget of a continuous
perception system. Even when updates are restricted to input prompts
or low-rank adapters, repeated optimization can remain expensive.
StreamTTO addresses this cost by reusing both visual representations
and accumulated adaptation across frames.

The underlying design principle is simple: freeze a causal feature
extractor, cache its representations, and carry a shared decoder and
its optimizer state forward while adapting to a short window of recent
observations. This reduces both repeated feature computation and the
additional optimization required for each incoming frame.
During decoder adaptation, cached features remain consistent with
recomputation under the same causal execution policy.
At fixed decoder parameters, the change between consecutive window
gradients is bounded by differences in the relevant per-frame gradients.
When the preceding window gradient is small and the relevant per-frame
gradients vary little, the inherited decoder also has a small gradient
on the new window. This provides a conditional explanation for reusing
adaptation across overlapping windows.

The experiments show the practical value of this reuse. On the TE04
development sequence, warm-start adaptation with a three-frame window
and two passes achieves accuracy comparable to reset adaptation with
20 passes, using one tenth as many optimizer updates per subsequent
frame. The diagnostic experiments further show that additional updates
can improve the fit to sparse observations while degrading accuracy at
unobserved locations. These findings support a small update budget once
adaptation state is carried across frames.

On MS-Depth at $518\times392$ resolution, StreamTTO achieves
15.73~FPS with VGGT and 25.90~FPS with MoGe-2 on an NVIDIA B200 GPU.
The MoGe-2 variant achieves approximately $47.7\times$, $60.8\times$,
and $15.7\times$ the throughput of TestPromptDC, Marigold-DC, and
offline CAPA-LoRA, respectively, while attaining the lowest RMSE
among the evaluated methods across all four MS-Depth conditions.
These results support the feasibility of sustained sparse-depth-guided
TTO for streaming perception under the tested resolution and hardware.
More broadly, jointly reusing representations and adaptation state may
benefit other tasks with causal features, efficiently adaptable
prediction heads, and sparse online supervision. Online 3D
reconstruction with intermittent range measurements and pose estimation
from sparse correspondences are promising directions for testing
this principle.

\subsection{StreamTTO Implementation Details}

\paragraph{Implementation Details for Adaptation Ablations.}
\label{app:adaptation_ablation_details}

For the adaptation-target and parameterization ablations in
Tables~\ref{tab:abl_target} and~\ref{tab:abl_lora}, we tune
the learning rate separately for each variant on the MS-Depth
TE04 development sequence.
This includes encoder-only, decoder-only, and joint image-encoder
and decoder parameter updates, as well as their LoRA counterparts.

For LoRA, we closely follow the CAPA-LoRA~\citep{ke2026depth}
configuration used with VGGT, using AdamW with rank $r=4$
and scaling parameter $\alpha=8$.
The selected learning rate is $10^{-3}$ for the encoder-LoRA,
decoder-LoRA, and full-LoRA variants in
Table~\ref{tab:abl_lora}.
Direct updates to pretrained parameters use smaller learning
rates, selected separately for each adaptation target;
the default decoder-only configuration uses $3\times10^{-5}$.
The LoRA variants retain the causal online adaptation protocol
used in our ablations, with the same sparse observations,
CSWO window size, and number of optimization passes.
The comparisons therefore use separately tuned learning rates
while keeping the data and evaluation protocol consistent.

\paragraph{Offline fine-tuning.}
For our VGGT-based implementation, we follow the causal
fine-tuning procedure of StreamVGGT~\citep{zhuo2026streaming},
starting from the public VGGT-1B
checkpoint~\citep{wang2025vggt}.
The VGGT backbone uses $L=24$ spatial and temporal attention
layers and FlashAttention-2 for efficient inference.
We freeze the image encoder and depth decoder and fine-tune
the attention components for causal processing with cached
keys and values.
We train for 10 epochs using AdamW, with a 0.5-epoch linear
warm-up followed by cosine decay and a peak learning rate
of $10^{-6}$.
Each iteration randomly samples a batch of 10 frames from
diverse training scenes, preserving variable aspect ratios
while resizing the longest image edge to 518 pixels.
Fine-tuning the VGGT-based model takes approximately one day
on 16 NVIDIA B200 GPUs.

For MoGe-2~\citep{wang2026moge}, we insert causal attention
layers between the pretrained image encoder and depth decoder.
These layers aggregate current and past frame features using
causal attention, with key--value (KV) caching for efficient
streaming inference, and pass the resulting features to
the depth decoder.
During offline fine-tuning, we freeze both the pretrained
image encoder and depth decoder and train only the newly
added causal attention layers, using the same optimization
settings and training data as the VGGT-based variant.

\paragraph{Training data.}
For both backbones, we use the same 13-dataset collection
as StreamVGGT~\citep{zhuo2026streaming}:
Co3Dv2~\citep{reizenstein2021common}, BlendedMVS~\citep{yao2020blendedmvs}, ARKitScenes~\citep{baruch2021arkitscenes}, MegaDepth~\citep{li2018megadepth}, WildRGB-D~\citep{xia2024rgbd}, ScanNet~\citep{dai2017scannet},
HyperSim~\citep{roberts2021hypersim}, OmniObject3D~\citep{wu2023omniobject3d}, MVS-Synth~\citep{huang2018deepmvs}, PointOdyssey~\citep{zheng2023pointodyssey}, Virtual KITTI~\citep{cabon2020virtual}, Spring~\citep{mehl2023spring}, and Waymo~\citep{sun2020scalability}.
The collection covers indoor and outdoor environments and
diverse temporal scales, balancing synthetic data and
real-world captures.
During subsequent online adaptation, each variant freezes
its entire feature extractor, including the image encoder
and causal attention layers, and optimizes only its depth
decoder, shared across frames, using sparse observations.

\paragraph{MoGe-2 Implementation Details.}
For MoGe-2, we insert three spatial attention layers followed by
three causal temporal attention layers between the pretrained image
encoder and depth decoder.
All six layers use eight attention heads and preserve the feature
dimension of the encoder output.
The added layers are randomly initialized, and we introduce no
additional positional encoding in these layers.
The causal temporal layers incorporate current and past visual context
using a 6-frame context, comprising the current frame and five preceding
frames, with key--value caching for streaming inference.
During offline fine-tuning, we freeze the pretrained image encoder and
depth decoder and train only the six added attention layers using the
original MoGe-2 training loss.
We use the same optimization settings and training data as the
VGGT-based variant.
During online test-time adaptation, we freeze the image encoder and
all attention layers and update only the shared depth decoder using
the sparse-depth objective described in Section~4.2.

\paragraph{Runtime measurement.}
All runtimes are measured on an NVIDIA B200 GPU.
For each dataset, FPS is computed as the total number of evaluated
frames divided by the total processing time, including StreamTTO's
initial 20 bootstrap updates.
Table~1 reports the average FPS across the five public benchmarks
at their dataset-specific input resolutions, including
$1216 \times 352$ for KITTI and $1936 \times 1216$ for DDAD.
Table~2 reports FPS on MS-Depth at $518 \times 392$.
The higher-resolution inputs in the public benchmarks contribute
to the lower average FPS in Table~1 compared with Table~2.

\begin{figure*}[t!]
    \centering
    \includegraphics[width=1.0\linewidth]{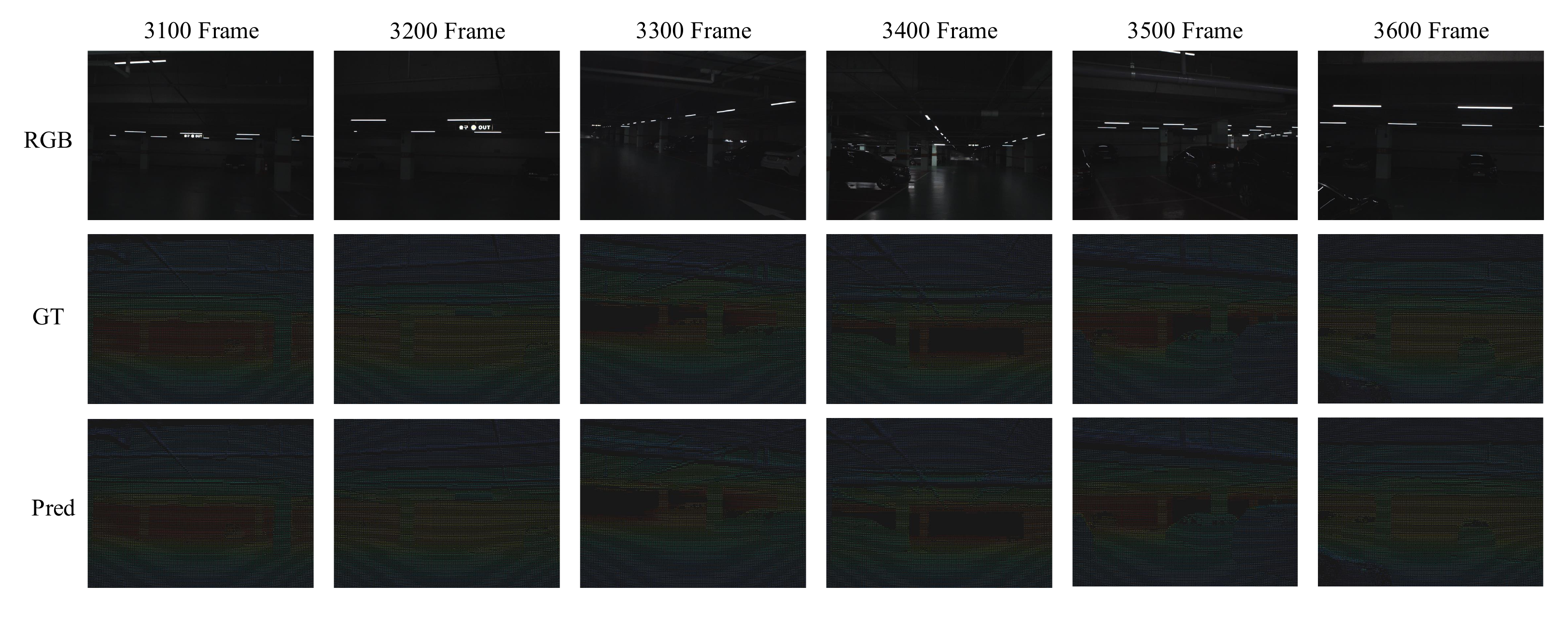}
   \caption{Qualitative results on MS-Depth at 100-frame intervals, showing RGB images, ground-truth depth, and StreamTTO predictions in that order.}
    \label{fig:sup_vis1}
\end{figure*}

\begin{figure*}[t!]
    \centering
    \includegraphics[width=1.0\linewidth]{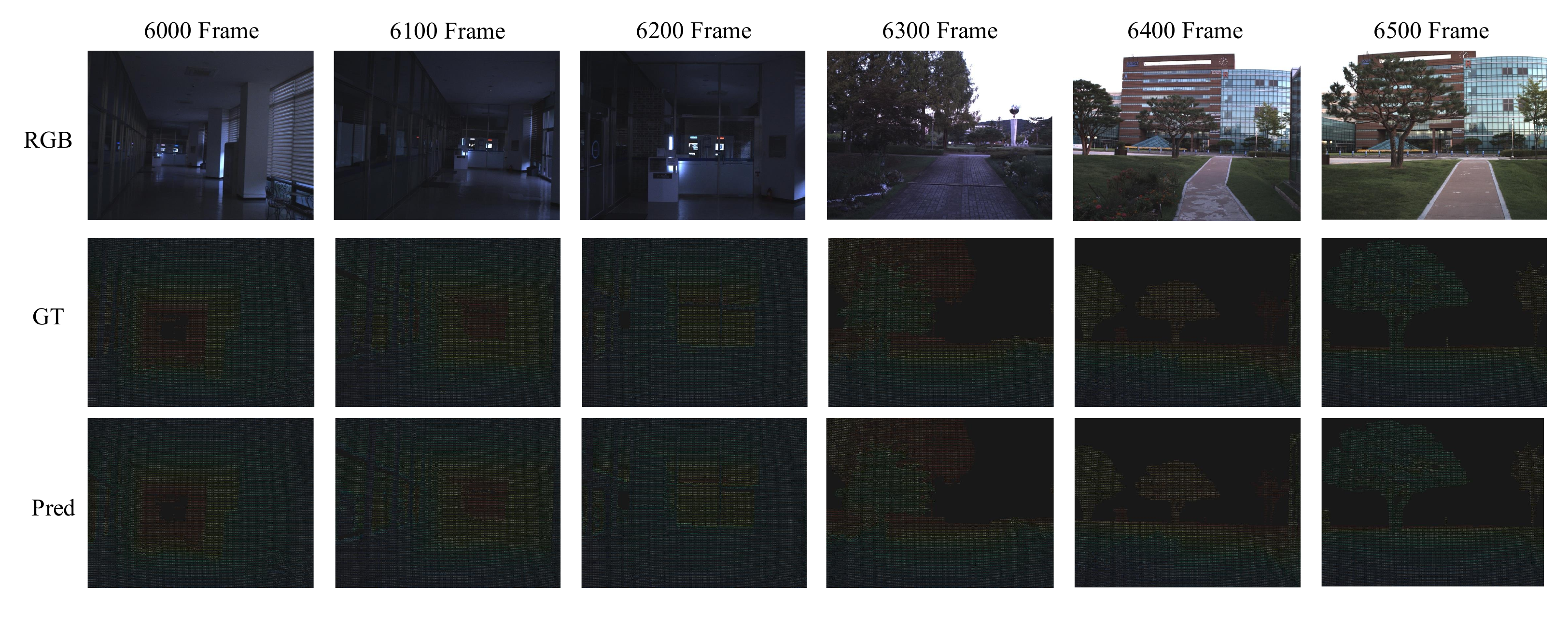}
\caption{Qualitative results on MS-Depth at 100-frame intervals, showing RGB images, ground-truth depth, and StreamTTO predictions in that order.}
\label{fig:sup_vis2}
\end{figure*}

\subsection{Qualitative Results}
Figures~\ref{fig:sup_vis1} and~\ref{fig:sup_vis2} present qualitative results on MS-Depth at 100-frame intervals.
StreamTTO recovers scene geometry consistent with the reference depth across indoor, underground, and outdoor environments, despite changes in illumination and depth range.
These examples complement the quantitative results, illustrating the robustness of online adaptation under changing observation conditions.

\subsection{Interaction between Window Size and Optimization Passes}
\label{sec:window_pass_interaction}

We extend the initialization and window-size ablations in
Tables~\ref{tab:abl_init} and~\ref{tab:abl_window} by jointly
varying the CSWO window size $W$ and optimization passes $R$
on MS-Depth TE04 using the VGGT backbone.
We compare reset adaptation, which reinitializes the decoder
and optimizer at every incoming frame, with warm-start
adaptation, which retains both states across frames after
initial bootstrap adaptation.
Each pass processes the available causal window as one batch
and applies one optimizer update.
Tables~\ref{tab:window_pass_reset} and~\ref{tab:window_pass_warm}
report RMSE for the two initialization strategies.

\begin{table}[t]
    \centering
    \small
    \captionsetup{position=bottom,skip=6pt}
    \setlength{\tabcolsep}{5pt}
    \renewcommand{\arraystretch}{1.10}
    \begin{tabular*}{\linewidth}{@{\extracolsep{\fill}}lrrrrrr@{}}
        \toprule
        & \multicolumn{6}{c}{Optimization passes $R$} \\
        \cmidrule(l){2-7}
        $W$ & 1 & 2 & 5 & 10 & 20 & 40 \\
        \midrule
        1  & 1.4609 & 1.3786 & 1.1952 & 0.9908
           & \underline{0.9245} & 0.9623 \\
        2  & 1.4583 & 1.3734 & 1.1847 & 0.9723
           & \underline{0.8895} & 0.9233 \\
        3  & 1.4577 & 1.3725 & 1.1831 & 0.9681
           & \underline{0.8753} & 0.9039 \\
        6  & 1.4588 & 1.3745 & 1.1883 & 0.9728
           & \underline{0.8615} & 0.8769 \\
        12 & 1.4623 & 1.3812 & 1.2028 & 0.9927
           & 0.8645 & \underline{0.8600} \\
        \bottomrule
    \end{tabular*}
    \caption{Window size and optimization passes under reset
    adaptation on MS-Depth TE04 with the VGGT backbone.
    The decoder and optimizer are reinitialized at every frame.
    Values are RMSE in meters (lower is better); underlined
    entries indicate the lowest RMSE in each row.}
    \label{tab:window_pass_reset}
\end{table}

\begin{table}[t]
    \centering
    \small
    \captionsetup{position=bottom,skip=6pt}
    \setlength{\tabcolsep}{5pt}
    \renewcommand{\arraystretch}{1.10}
    \begin{tabular*}{\linewidth}{@{\extracolsep{\fill}}lrrrrrr@{}}
        \toprule
        & \multicolumn{6}{c}{Optimization passes $R$} \\
        \cmidrule(l){2-7}
        $W$ & 1 & 2 & 5 & 10 & 20 & 40 \\
        \midrule
        1  & 0.9733 & 0.9314 & \underline{0.9038}
           & 0.9781 & 0.9770 & 0.9758 \\
        2  & 0.9481 & 0.9004 & \underline{0.8869}
           & 0.9514 & 1.5399 & 1.8283 \\
        3  & 0.9233 & 0.8694 & \underline{0.8594}
           & 0.9704 & 0.9846 & 1.2465 \\
        6  & 0.9048 & \underline{0.8551} & 0.8812
           & 0.9746 & 1.0659 & 1.2872 \\
        12 & \underline{0.8922} & 0.9544 & 1.0273
           & 1.2603 & 1.4860 & 1.5117 \\
        \bottomrule
    \end{tabular*}
    \caption{Window size and optimization passes under warm-start
    adaptation on MS-Depth TE04 with the VGGT backbone.
    The decoder and optimizer states are retained across frames,
    with $R$ passes per subsequent frame after initial bootstrap
    adaptation. Values are RMSE in meters (lower is better);
    underlined entries indicate the lowest RMSE in each row.}
    \label{tab:window_pass_warm}
\end{table}

Under reset adaptation, larger windows improve the best
attainable accuracy when sufficient optimization is allowed
(Table~\ref{tab:window_pass_reset}).
The lowest RMSE among the evaluated pass counts decreases
from 0.9245~m at $W=1$ to 0.8600~m at $W=12$.
These minima occur at $R=20$ for $W\in\{1,2,3,6\}$ and
at $R=40$ for $W=12$.
In contrast, warm-start adaptation reaches its lowest RMSE
with substantially fewer passes
(Table~\ref{tab:window_pass_warm}).
The best-performing pass count decreases as the window grows:
$R=5$ for $W\in\{1,2,3\}$, $R=2$ for $W=6$, and $R=1$
for $W=12$.
At $W=3$, warm-start adaptation with two passes achieves
0.8694~m, comparable to the 0.8753~m obtained by reset
adaptation with 20 passes, using one tenth as many optimizer
updates per subsequent frame.

Increasing the update budget beyond these best-performing
pass counts degrades warm-start accuracy.
For every evaluated multi-frame window ($W\geq2$), RMSE
increases at each tested pass count beyond the row-wise minimum.
For example, at $W=2$, RMSE increases from 0.8869~m at
$R=5$ to 1.5399~m at $R=20$ and 1.8283~m at $R=40$.
At $W=12$, the lowest RMSE is 0.8922~m with a single pass,
increasing to 0.9544~m with two passes and 1.5117~m with
40 passes.
Thus, under the tested settings, additional optimization
does not necessarily improve the accuracy of a decoder
whose adaptation state is already carried across frames.

These trends are consistent with a reduced need for additional
optimization when accumulated adaptation and overlapping
window constraints are reused.
Consecutive full windows share $W-1$ observations, so the
inherited decoder has already been adapted to much of the
next window's supervision.
Window size also changes the balance between current and
past observations.
Under the weighting scheme in
Section~\ref{sec:decoder_adaptation}, the current frame
receives weight two and each past frame receives weight one.
For a full window, the current frame therefore has a
coefficient of $2/(W+1)$ in the normalized loss.
Larger windows consequently place relatively less weight
on the current frame, and additional updates may favor
fitting recent observations without improving current-frame
accuracy.
Accumulated adaptation, overlapping supervision, and this
change in relative weighting provide possible explanations
for the observed interaction between $W$ and $R$.

We select $W=3$ and $R=2$ as the default configuration to
balance accuracy and computational cost.
At $W=3$, increasing $R$ from two to five reduces RMSE
from 0.8694 to 0.8594~m, an improvement of approximately
1.15\%, but requires $2.5\times$ as many optimizer updates
per frame.
Similarly, increasing $W$ from three to six at $R=2$
reduces RMSE from 0.8694 to 0.8551~m, an improvement of
approximately 1.6\%, while increasing total FLOPs per frame
by 57.8\% (Table~\ref{tab:abl_window}).
The default configuration therefore achieves near-best
accuracy in this sweep with a smaller per-frame
computational budget.

\subsection{Diagnostic Analysis of Optimization Passes and Window Size}
\label{sec:adaptation_diagnostics}

We investigate the degradation associated with 20-pass warm-start
adaptation and a larger optimization window on the full
MS-Depth TE04 development sequence.
Table~\ref{tab:adaptation_diagnostics} compares ten configurations,
including learning-rate changes, periodic resets, and current-frame
reweighting.
For these diagnostic experiments, the baseline assigns equal
weight to all frames in the optimization window.
We report mean per-frame RMSE at sparse-input locations and at
the remaining valid evaluation pixels separately.

\begin{table}[t]
    \centering
    \small
    \captionsetup{position=bottom,skip=6pt}
    \setlength{\tabcolsep}{4pt}
    \renewcommand{\arraystretch}{1.10}
    \begin{tabular*}{\linewidth}{@{\extracolsep{\fill}}rrlrr@{}}
        \toprule
        & & & \multicolumn{2}{c}{RMSE (m) $\downarrow$} \\
        \cmidrule(l){4-5}
        $W$ & $R$ & Configuration & Observed & Unobserved \\
        \midrule
        3 & 2 & Baseline
          & 0.8162 & \underline{0.8696} \\
        \midrule
        3 & 20 & No modification
          & \underline{0.2623} & 0.9773 \\
        3 & 20 & Lower constant LR
          & 0.7454 & 0.9321 \\
        3 & 20 & LR decay
          & 0.4113 & 0.9955 \\
        3 & 20 & Optimizer-only reset
          & 0.2740 & 0.9869 \\
        3 & 20 & Decoder and optimizer reset
          & 0.2823 & 0.9603 \\
        \midrule
        12 & 2 & No modification
           & 0.8696 & 0.9251 \\
        12 & 2 & LR decay
           & 0.8926 & 0.9309 \\
        12 & 2 & Decoder and optimizer reset
           & 0.8664 & 0.9151 \\
        12 & 2 & Current-frame weight 5.5
           & 0.7907 & 0.8734 \\
        \bottomrule
    \end{tabular*}
    \caption{Diagnostic experiments on the full MS-Depth TE04
    development sequence.
    Observed and unobserved RMSE are evaluated at sparse-input
    locations and at the remaining valid evaluation pixels,
    respectively; underlined values indicate the lowest RMSE
    in each column.
    All window frames receive weight 1, except in the W12
    reweighting variant, which assigns weight 5.5 to the current
    frame and 1 to each past frame.
    Periodic resets are applied every 250 frames throughout
    the sequence, immediately before adaptation at zero-based
    frame indices $250,500,750,1000,\ldots$.
    Optimizer-only reset clears the AdamW state while retaining
    the decoder; full reset additionally restores the pretrained
    decoder.
    Both retain the KV cache, cached features, and optimization
    window, without additional bootstrap updates.
    All frames, including those immediately following resets,
    contribute to the reported mean per-frame RMSE.}
    \label{tab:adaptation_diagnostics}
\end{table}

At $W=3$, increasing the pass count from 2 to 20 reduces
observed RMSE from 0.8162 to 0.2623~m but increases unobserved
RMSE from 0.8696 to 0.9773~m
(Table~\ref{tab:adaptation_diagnostics}).
This is consistent with excessive fitting to sparse observations
that does not generalize to unobserved locations.
A lower constant learning rate partially mitigates the
degradation, reducing unobserved RMSE to 0.9321~m.
The tested learning-rate decay and optimizer-only reset do not
improve accuracy, while resetting both the decoder and optimizer
provides a modest improvement to 0.9603~m.

For $W=12$, learning-rate decay does not improve accuracy,
and resetting both the decoder and optimizer provides only
a modest improvement from 0.9251 to 0.9151~m.
In contrast, increasing the current-frame weight to 5.5
reduces unobserved RMSE to 0.8734~m.
For a full window, this changes the current-frame coefficient
in the normalized loss from $1/12$ to
$5.5/(5.5+11)=1/3$, matching the W3 baseline used in these
diagnostic experiments.
The resulting gap to the baseline is 0.0038~m, recovering
approximately 93.2\% of the original RMSE gap.
This supports the contribution of current-frame weighting to
the degradation observed with the larger window under the
tested conditions.

The accuracy recovery at $W=12$ retains the computational
cost of processing a larger window.
The computational analysis in Table~\ref{tab:abl_window}
shows that increasing $W$ from 3 to 12 at $R=2$ raises FLOPs
per frame from 7.028 to 19.217~T, approximately a
$2.73\times$ increase.
Reweighting does not reduce the number of frames processed
by the decoder.
On TE04, the reweighted W12 configuration also does not
improve upon the W3 baseline RMSE of 0.8696~m.
These results support our choice of $W=3$ and $R=2$, which
provides comparable accuracy at substantially lower
computational cost.

\subsection{Long-Horizon Adaptation Stability}
\label{sec:long_horizon_stability}

We evaluate the effect of accumulated adaptation on the 12
MS-Depth evaluation sequences, excluding the TE04 development
sequence.
Both experiments use the VGGT backbone with a 6-frame causal KV
context, a CSWO window of $W=3$, and $R=2$ window passes per
incoming frame after initialization.
Each sequence begins with $K_{\mathrm{boot}}=20$ bootstrap updates.
Accuracy is evaluated at valid reference-depth pixels, excluding
the sparse-input observation locations.

\paragraph{Environmental transitions after accumulated adaptation.}
We examine whether adaptation accumulated earlier in a recording
remains effective after environmental transitions.
Using the same 51 transition clips as in
Table~\ref{tab:transition_matched_flops}, we compare two initialization
strategies.
\emph{Local initialization} resets the decoder and optimizer
approximately four seconds before each transition and performs
20 bootstrap updates.
\emph{Full-stream warm-start} processes the recording from its
first frame and carries the decoder and optimizer states through
all subsequent frames without resetting at transitions.
Both configurations use the same visual features extracted from
the full recording and the same sparse observations.
Local initialization retains the KV and feature caches,
isolating the effect of the decoder and optimizer adaptation
history.
We report RMSE over the post-transition intervals
$[0,1)$, $[1,3)$, and $[3,10)$ seconds, together with the
overall result over $[0,10)$ seconds.

\begin{table}[t]
    \centering
    \small
    \setlength{\tabcolsep}{5pt}
    \renewcommand{\arraystretch}{1.08}
    \caption{
        Post-transition RMSE (m, $\downarrow$) on 51 transition
        clips from the 12 MS-Depth evaluation sequences,
        excluding TE04.
        Local initialization resets the decoder and optimizer
        approximately four seconds before each transition and
        performs 20 bootstrap updates.
        Full-stream warm-start retains adaptation state from
        the beginning of each recording.
        Both configurations use identical visual features
        and sparse observations.
    }
    \label{tab:transition_history}
    \begin{tabular}{lcccc}
        \toprule
        Initialization
        & $[0,1)$ s
        & $[1,3)$ s
        & $[3,10)$ s
        & $[0,10)$ s \\
        \midrule
        Local reset + bootstrap
        & 1.26 & 1.16 & 1.11 & 1.14 \\
        Full-stream warm-start
        & 1.24 & 1.10 & 1.07 & 1.09 \\
        \bottomrule
    \end{tabular}
\end{table}

Table~\ref{tab:transition_history} shows that full-stream
warm-start achieves lower RMSE in all three post-transition
intervals, reducing overall RMSE from 1.14 to 1.09~m.
The improvement is present during the first second after the
transition and persists over the subsequent intervals.
These results suggest that corrections accumulated earlier
in the stream remain useful after environmental transitions.

\paragraph{Continuous adaptation versus periodic resets.}
We further compare continuous warm-start with periodic
initialization throughout each of the 12 complete recordings.
The continuous configuration retains the decoder and optimizer
states for the entire sequence.
The periodic-reset configurations reinitialize both states
every 250 frames while retaining the KV cache, cached features,
and optimization window.
We evaluate periodic resets both without additional bootstrap
updates and with 20 additional bootstrap updates at each reset.
All configurations otherwise use the same $W=3$ and $R=2$
adaptation settings.
All frames, including those immediately following resets,
contribute to the reported metrics.
We report RMSE in elapsed-time intervals of 0--5, 5--10,
10--15, and at least 15 minutes, as well as over the complete
recordings.

\begin{table}[t]
    \centering
    \small
    \setlength{\tabcolsep}{4pt}
    \renewcommand{\arraystretch}{1.08}
    \caption{
        Long-horizon adaptation on the 12 complete MS-Depth
        evaluation sequences, excluding TE04.
        RMSE (m, $\downarrow$) is reported by elapsed time
        and over the complete recordings.
        Periodic resets reinitialize the decoder and optimizer
        every 250 frames while retaining the KV and feature
        caches.
        The bootstrap variant performs 20 additional updates
        at each reset.
        All configurations use $W=3$ and $R=2$ after bootstrap.
    }
    \label{tab:long_horizon_reset}
    \begin{tabular}{lccccc}
        \toprule
        Adaptation strategy
        & 0--5 min
        & 5--10 min
        & 10--15 min
        & $\geq$15 min
        & Overall \\
        \midrule
        Continuous warm-start
        & 1.06 & 1.14 & 1.05 & 1.11 & 1.09 \\
        Periodic reset
        & 1.15 & 1.24 & 1.16 & 1.22 & 1.19 \\
        Periodic reset + bootstrap
        & 1.05 & 1.16 & 1.08 & 1.14 & 1.11 \\
        \bottomrule
    \end{tabular}
\end{table}

As shown in Table~\ref{tab:long_horizon_reset}, continuous
warm-start achieves an overall RMSE of 1.09~m, compared with
1.19~m for periodic resets without additional bootstrap and
1.11~m with bootstrap.
Although periodic resets with bootstrap yield slightly lower
RMSE during the first five minutes, continuous adaptation
performs better in all subsequent intervals.
Because scene difficulty changes throughout each recording,
we assess accumulated adaptation through comparisons with
reset baselines on the same frames.
The relative performance of continuous adaptation remains
competitive in the later intervals, supporting the retention
of adaptation state throughout the evaluated recordings.
Together with the transition results in
Table~\ref{tab:transition_history}, these findings support
the effectiveness of full-stream warm-start under the
evaluated recording durations and environmental changes.

\subsection{Contribution Beyond the Backbone}
\label{sec:shared_backbone_adaptation}

We examine the contribution of StreamTTO through two complementary
comparisons: adaptation strategies on the same causal backbone,
and CSWO with and without causal temporal attention.
All variants receive identical RGB frames and sparse observations
and follow the same evaluation protocol.

First, we compare TestPromptDC~\citep{jeong2025test} and StreamTTO
using the same causal VGGT and MoGe-2 backbones.
Within each backbone group, the two methods start from the same
pretrained causal-backbone checkpoint and use a 6-frame causal
KV context, comprising the current frame and five preceding frames.
The corresponding Causal-Attn row reports the causal backbone
without test-time parameter updates.

For TestPromptDC, we use the official implementation%
\footnote{\url{https://github.com/JinhwiPark/TestPromptDC/blob/main/main.py}}
and retain its original optimization configuration for both
causal backbones.
At each frame, an additive three-channel visual prompt is initialized
to zero and optimized for 50 updates using a fresh AdamW optimizer
with a learning rate of $2\times10^{-3}$.
The objective is the sum of the L1 and RMSE losses at sparse
observation locations, with least-squares scale--shift alignment
recomputed at each iteration.
Only the input prompt is updated; all backbone weights remain fixed.
StreamTTO instead freezes the feature extractor and updates the
shared depth decoder using CSWO with $W=3$ and $R=2$ after bootstrap,
retaining decoder parameters and optimizer state across frames.

Second, the CSWO-only variants apply the same window-based decoder
adaptation without causal temporal attention.
They extract visual features independently for each frame,
without a temporal KV cache, while retaining the feature buffer
used for CSWO.
They use the same bootstrap and decoder adaptation settings as
StreamTTO, including $W=3$, $R=2$, and the retention of decoder
parameters and optimizer state across frames.
Thus, CSWO still reuses recent observations and adaptation state,
but the visual features do not incorporate past frames through
causal temporal attention.

Tables~\ref{tab:sub_tab1} and~\ref{tab:sub_tab2} show that
StreamTTO's advantages persist when compared with TestPromptDC
on the same causal backbone.
On the public benchmarks, StreamTTO achieves lower RMSE than
TestPromptDC on three of five datasets with VGGT and four
of five datasets with MoGe-2, with throughput speedups of
$35.5\times$ and $117.5\times$, respectively.
On MS-Depth, StreamTTO achieves lower RMSE and MAE in all
four environmental conditions with both backbones,
with throughput speedups of $25.9\times$ for VGGT and
$60.5\times$ for MoGe-2.
These results show that backbone choice alone does not explain
the observed accuracy and efficiency gains.

Full StreamTTO also achieves lower RMSE and MAE than the
CSWO-only variants on every public benchmark and in every
MS-Depth environmental condition, for both backbones.
Although the CSWO-only variants run faster, incorporating causal
temporal attention substantially improves depth accuracy.
These findings support combining temporal visual context with
the reuse of recent observations and decoder adaptation state:
under the evaluated settings, CSWO alone does not recover the
accuracy of the complete framework.

\begin{table*}[t]
\centering
\setlength{\tabcolsep}{3.0pt}
\renewcommand{\arraystretch}{1.08}
\setlength{\heavyrulewidth}{1.2pt}
\setlength{\lightrulewidth}{0.55pt}
\setlength{\cmidrulewidth}{0.45pt}

\resizebox{\textwidth}{!}{%
\begin{tabular}{@{}l*{5}{cc}c@{}}
\toprule

\multirow{2}{*}{\textbf{Method}}
& \multicolumn{2}{c}{\textbf{KITTI}}
& \multicolumn{2}{c}{\textbf{Bonn}}
& \multicolumn{2}{c}{\textbf{NYUv2 Raw}}
& \multicolumn{2}{c}{\textbf{TUM RGB-D}}
& \multicolumn{2}{c}{\textbf{DDAD}}
& \textbf{Speed} \\

\cmidrule(lr){2-3}
\cmidrule(lr){4-5}
\cmidrule(lr){6-7}
\cmidrule(lr){8-9}
\cmidrule(lr){10-11}
\cmidrule(lr){12-12}

& \textbf{RMSE}$\downarrow$ & \textbf{MAE}$\downarrow$
& \textbf{RMSE}$\downarrow$ & \textbf{MAE}$\downarrow$
& \textbf{RMSE}$\downarrow$ & \textbf{MAE}$\downarrow$
& \textbf{RMSE}$\downarrow$ & \textbf{MAE}$\downarrow$
& \textbf{RMSE}$\downarrow$ & \textbf{MAE}$\downarrow$
& \textbf{Avg. FPS}$\uparrow$ \\

\midrule

VGGT + Causal-Attn
& 3.0817 & 1.3334
& 0.1723 & 0.0792
& 0.1602 & 0.0832
& 0.2150 & 0.0988
& 12.7059 & 6.6541
& \underline{17.52} \\

VGGT + Causal-Attn + TestPromptDC
& 1.2449 & 0.3768
& \underline{0.0559} & \underline{0.0208}
& 0.0584 & 0.0254
& \textbf{0.0550} & \textbf{0.0190}
& 9.3939 & 2.6467
& 0.201 \\

VGGT + CSWO
& 2.2404 & 0.8587
& 0.1330 & 0.0543
& 0.0992 & 0.0621
& 0.1425 & 0.0782
& 8.9450 & 4.2453
& 9.491 \\

\rowcolor{gray!15}VGGT + StreamTTO
& \underline{1.0223} & 0.3920
& 0.0592 & 0.0210
& \underline{0.0403} & \underline{0.0220}
& 0.0738 & 0.0274
& 7.1448 & \underline{2.3483}
& 7.14 \\

\midrule

MoGe-2 + Causal-Attn
& 2.8130 & 1.2086
& 0.1641 & 0.0793
& 0.1194 & 0.0732
& 0.2414 & 0.1075
& 10.2409 & 4.1452
& \textbf{33.28} \\

MoGe-2 + Causal-Attn + TestPromptDC
& 1.2212 & \underline{0.3744}
& 0.0602 & 0.0213
& 0.0533 & 0.0247
& \underline{0.0577} & \underline{0.0195}
& 9.2037 & 2.6379
& 0.102 \\

MoGe-2 + CSWO
& 1.8369 & 0.7469
& 0.1122 & 0.0500
& 0.0795 & 0.0446
& 0.1550 & 0.0745
& \underline{6.4210} & 2.7939
& 14.83 \\

\rowcolor{gray!15}MoGe-2 + StreamTTO
& \textbf{0.7796} & \textbf{0.2763}
& \textbf{0.0430} & \textbf{0.0155}
& \textbf{0.0316} & \textbf{0.0173}
& 0.0586 & \underline{0.0195}
& \textbf{5.2864} & \textbf{1.7124}
& 11.99 \\

\bottomrule
\end{tabular}%
}

\caption{
Contribution of the adaptation strategy and causal temporal attention
across five public benchmarks.
All variants receive identical RGB frames and sparse observations
and follow the same evaluation protocol.
Within each backbone group, Causal-Attn, Causal-Attn + TestPromptDC,
and StreamTTO share the same pretrained causal-backbone initialization
and 6-frame causal KV context.
Causal-Attn performs no test-time parameter updates;
TestPromptDC optimizes an input visual prompt independently per frame.
CSWO applies window-based decoder adaptation without causal temporal
attention or a temporal KV cache.
Both CSWO and StreamTTO use $W=3$ and $R=2$ after bootstrap
and retain decoder parameters and optimizer state across frames.
CAPA-LoRA is excluded because it uses an offline clip-level
protocol in our evaluation.
RMSE and MAE are reported in meters, excluding sparse-input
locations, following Table~\ref{tab:main_depth_results_with_fps}.
FPS is averaged over the five benchmarks at their dataset-specific
input resolutions on one B200.
Bold and underlined values indicate the best and second-best
results within this table, including ties.
}
\label{tab:sub_tab1}
\end{table*}

\begin{table*}[t]
\centering
\setlength{\tabcolsep}{3.2pt}
\renewcommand{\arraystretch}{1.10}

\resizebox{\textwidth}{!}{%
\begin{tabular}{@{}l*{4}{cc}c@{}}
\toprule[1.2pt]

\multicolumn{10}{c}{\textbf{MS-Depth}} \\
\midrule

\multirow{3}{*}{\textbf{Method}}
& \multicolumn{4}{c}{\textbf{Indoor}}
& \multicolumn{4}{c}{\textbf{Outdoor}}
& \multirow{3}{*}{\textbf{FPS}$\uparrow$} \\

\cmidrule(lr){2-5}
\cmidrule(lr){6-9}

& \multicolumn{2}{c}{\textbf{Underground}}
& \multicolumn{2}{c}{\textbf{Aboveground}}
& \multicolumn{2}{c}{\textbf{Day}}
& \multicolumn{2}{c}{\textbf{Night}}
& \\

\cmidrule(lr){2-3}
\cmidrule(lr){4-5}
\cmidrule(lr){6-7}
\cmidrule(lr){8-9}

& \textbf{RMSE}$\downarrow$ & \textbf{MAE}$\downarrow$
& \textbf{RMSE}$\downarrow$ & \textbf{MAE}$\downarrow$
& \textbf{RMSE}$\downarrow$ & \textbf{MAE}$\downarrow$
& \textbf{RMSE}$\downarrow$ & \textbf{MAE}$\downarrow$
& \\

\midrule

VGGT + Causal-Attn
& 1.5222 & 0.9521
& 1.8435 & 0.6425
& 2.3414 & 1.4592
& 2.3252 & 1.3433
& \underline{39.84} \\

VGGT + Causal-Attn + TestPromptDC
& 1.0995 & 0.3174
& 1.2347 & 0.2536
& 1.8212 & 0.5253
& 1.7266 & 0.4693
& 0.608 \\

VGGT + CSWO
& 1.1234 & 0.6227
& 1.4582 & 0.4388
& 1.6999 & 0.8930
& 1.7881 & 0.8691
& 19.46 \\

\rowcolor{gray!15}VGGT + StreamTTO
& \underline{0.8077} & \underline{0.2610}
& \underline{1.0728} & \underline{0.2211}
& \underline{1.3666} & \underline{0.3668}
& \underline{1.2603} & \underline{0.3497}
& 15.73 \\

\midrule

MoGe-2 + Causal-Attn
& 1.5088 & 0.7783
& 1.9442 & 0.7325
& 2.7352 & 1.5499
& 2.7917 & 1.6565
& \textbf{70.25} \\

MoGe-2 + Causal-Attn + TestPromptDC
& 1.1233 & 0.3490
& 1.2752 & 0.2717
& 1.7801 & 0.4670
& 1.6584 & 0.4842
& 0.428 \\

MoGe-2 + CSWO
& 1.0184 & 0.5783
& 1.2210 & 0.5757
& 1.8025 & 1.1066
& 1.7783 & 1.2788
& 33.22 \\

\rowcolor{gray!15}MoGe-2 + StreamTTO
& \textbf{0.5899} & \textbf{0.2063}
& \textbf{0.8339} & \textbf{0.1647}
& \textbf{0.9950} & \textbf{0.2779}
& \textbf{0.8846} & \textbf{0.2558}
& 25.90 \\

\bottomrule[1.2pt]
\end{tabular}%
}

\caption{
Contribution of the adaptation strategy and causal temporal attention
on the 12 MS-Depth test sequences excluding TE04.
All variants receive identical RGB frames and sparse observations
and follow the same evaluation protocol.
Within each backbone group, Causal-Attn, Causal-Attn + TestPromptDC,
and StreamTTO share the same pretrained causal-backbone initialization
and 6-frame causal KV context.
Causal-Attn performs no test-time parameter updates;
TestPromptDC optimizes an input visual prompt independently per frame.
CSWO applies window-based decoder adaptation without causal temporal
attention or a temporal KV cache.
Both CSWO and StreamTTO use $W=3$ and $R=2$ after bootstrap
and retain decoder parameters and optimizer state across frames.
CAPA-LoRA is excluded because it uses an offline clip-level
protocol in our evaluation.
RMSE and MAE are reported in meters, excluding sparse-input
locations; FPS is measured at $518\times392$ on one B200.
Bold and underlined values indicate the best and second-best
results within this table, including ties.
}
\label{tab:sub_tab2}
\end{table*}

\subsection{Robustness to Reduced Observation Density}
\label{sec:observation_density}

We evaluate robustness to reduced sparse-depth supervision on NYUv2 Raw
using observation densities of 3\% and 1\% of valid reference-depth pixels.
We evaluate the VGGT-based StreamTTO variant and follow the main
evaluation protocol: all methods receive identical sparse observations
at each density, and accuracy is evaluated at the remaining valid
reference-depth locations.

\begin{table*}[t]
\centering
\small
\setlength{\tabcolsep}{4pt}
\renewcommand{\arraystretch}{1.08}

\resizebox{0.5\textwidth}{!}{%
\begin{tabular}{@{}l*{2}{cc}@{}}
\toprule
& \multicolumn{2}{c}{Observations 3\%}
& \multicolumn{2}{c}{Observations 1\%} \\
\cmidrule(lr){2-3}
\cmidrule(lr){4-5}

Method
& RMSE$\downarrow$ & MAE$\downarrow$
& RMSE$\downarrow$ & MAE$\downarrow$ \\
\midrule

\multicolumn{5}{l}{\emph{Depth foundation models}} \\
Depth Anything V2 Metric
& 0.1660 & 0.1060 & 0.1661 & 0.1061 \\
UniDepthV2
& 0.1327 & 0.0698 & 0.1328 & 0.0698 \\
Video Depth Anything Metric (online)
& 0.2514 & 0.1380 & 0.2516 & 0.1381 \\
\midrule

\multicolumn{5}{l}{\emph{3D foundation models}} \\
MoGe-2
& 0.1222 & 0.0707 & 0.1223 & 0.0708 \\
VGGT
& 0.1861 & 0.0886 & 0.1862 & 0.0888 \\
StreamVGGT
& 0.1587 & 0.0794 & 0.1588 & 0.0795 \\
ABot-Recon
& 0.2276 & 0.1174 & 0.2277 & 0.1175 \\
\midrule

\multicolumn{5}{l}{\emph{Depth completion models}} \\
PriorDA v1.1
& 0.0786 & 0.0316 & 0.0919 & 0.0360 \\
OMNI-DC v1.1
& \underline{0.0583} & \textbf{0.0163}
& 0.0775 & \textbf{0.0230} \\
Marigold-DC
& 0.0783 & 0.0323 & 0.0976 & 0.0368 \\
TestPromptDC
& \textbf{0.0541} & \underline{0.0235}
& \underline{0.0679} & 0.0265 \\
CAPA-LoRA (offline)
& 0.0680 & 0.0302 & 0.0694 & 0.0308 \\
\midrule

\multicolumn{5}{l}{\emph{Our variants}} \\
VGGT + Causal-Attn
& 0.1687 & 0.0858 & 0.1714 & 0.0886 \\
VGGT + StreamTTO (Ours)
& \textbf{0.0541} & 0.0237
& \textbf{0.0593} & \underline{0.0254} \\
\bottomrule
\end{tabular}%
}

\caption{Depth completion under reduced observation density on NYUv2 Raw.
Observation density is the fraction of valid reference-depth pixels
provided as input.
RMSE and MAE are reported in meters; lower is better.
Input observation locations are excluded from evaluation at each density.
StreamTTO uses the VGGT backbone.
Bold and underlined values indicate the best and second-best results
in each column, respectively, including ties.}
\label{tab:observation_density}
\end{table*}

Table~\ref{tab:observation_density} shows that StreamTTO maintains strong
accuracy as the available observations become sparser.
Reducing the observation density from 3\% to 1\%, a 66.7\% reduction in
observations, increases its RMSE from 0.0541 to 0.0593~m, a 9.6\% increase.
Its MAE increases from 0.0237 to 0.0254~m, or 7.2\%.
Under the same density reduction, the RMSE of TestPromptDC and OMNI-DC
increases by 25.5\% and 32.9\%, respectively.

StreamTTO ties for the lowest RMSE with TestPromptDC at 3\% observation
density and achieves the lowest RMSE among all evaluated methods at 1\%.
At 1\% density, its RMSE of 0.0593~m is 12.7\% lower than that of the
next-best method, TestPromptDC (0.0679~m).
OMNI-DC retains the lowest MAE, while StreamTTO achieves the
second-lowest MAE at this density.
These results show that the VGGT-based StreamTTO remains effective
with substantially fewer sparse observations, with its relative
RMSE advantage strengthening over the evaluated density range.

\begin{table*}[t]
\centering
\setlength{\tabcolsep}{4pt}
\renewcommand{\arraystretch}{1.1}
\resizebox{\textwidth}{!}{%
\begin{tabular}{@{}l*{5}{cc}@{}}
\toprule
& \multicolumn{2}{c}{KITTI}
& \multicolumn{2}{c}{Bonn}
& \multicolumn{2}{c}{NYUv2 Raw}
& \multicolumn{2}{c}{TUM RGB-D}
& \multicolumn{2}{c}{DDAD} \\
\cmidrule(lr){2-3}
\cmidrule(lr){4-5}
\cmidrule(lr){6-7}
\cmidrule(lr){8-9}
\cmidrule(lr){10-11}
Method
& RMSE$\downarrow$ & MAE$\downarrow$
& RMSE$\downarrow$ & MAE$\downarrow$
& RMSE$\downarrow$ & MAE$\downarrow$
& RMSE$\downarrow$ & MAE$\downarrow$
& RMSE$\downarrow$ & MAE$\downarrow$ \\
\midrule

TestPromptDC
& 1.1266 & 0.3441
& 0.0514 & 0.0190
& 0.0495 & 0.0229
& 0.0514 & 0.0163
& 8.2693 & 2.2167 \\

TestPromptDC + warm-start
& 1.0994 & 0.3368
& 0.0507 & 0.0175
& 0.0460 & 0.0196
& 0.0512 & 0.0161
& 8.1666 & 2.1951 \\

\quad Error reduction
& 2.41\% & 2.12\%
& 1.36\% & 7.89\%
& 7.07\% & 14.41\%
& 0.39\% & 1.23\%
& 1.24\% & 0.97\% \\

\midrule

CAPA-LoRA (offline)
& 1.0902 & 0.2684
& 0.0566 & 0.0179
& 0.0546 & 0.0219
& 0.0560 & 0.0144
& 7.6213 & 1.7592 \\

CAPA-LoRA + warm-start
& 1.0076 & 0.2374
& 0.0540 & 0.0156
& 0.0488 & 0.0185
& 0.0532 & 0.0123
& 7.5613 & 1.7009 \\

\quad Error reduction
& 7.58\% & 11.55\%
& 4.59\% & 12.85\%
& 10.62\% & 15.53\%
& 5.00\% & 14.58\%
& 0.79\% & 3.31\% \\

\midrule

VGGT + Causal-Attn
& 3.0817 & 1.3334
& 0.1723 & 0.0792
& 0.1602 & 0.0832
& 0.2150 & 0.0988
& 12.7059 & 6.6541 \\

VGGT + StreamTTO
& 1.0223 & 0.3920
& 0.0592 & 0.0210
& 0.0403 & 0.0220
& 0.0738 & 0.0274
& 7.1448 & 2.3483 \\

\quad Error reduction
& 66.83\% & 70.60\%
& 65.64\% & 73.48\%
& 74.84\% & 73.56\%
& 65.67\% & 72.27\%
& 43.77\% & 64.71\% \\

\bottomrule
\end{tabular}%
}
\caption{
Effect of warm-start adaptation on baseline methods,
with StreamTTO included for comparison.
All experiments follow the same evaluation settings as
Table~\ref{tab:main_depth_results_with_fps}.
}
\label{tab:public_warmstart_variants}
\end{table*}

\subsection{Effect of Warm-Start Adaptation on Baseline Methods}
\label{sec:warmstart_baselines}

We evaluate how warm-start adaptation affects existing test-time
optimization methods on the five public benchmarks.
Table~\ref{tab:public_warmstart_variants} compares TestPromptDC
and CAPA-LoRA with their warm-start variants under the same
evaluation settings as Table~\ref{tab:main_depth_results_with_fps}.
We also include VGGT + Causal-Attn and VGGT + StreamTTO
to show the combined effect of our complete adaptation framework
on the causal backbone.

Warm-start adaptation improves both RMSE and MAE on all five
benchmarks for TestPromptDC and CAPA-LoRA.
Averaged over the benchmark-wise relative reductions,
RMSE improves by 2.50\% for TestPromptDC and 5.72\% for CAPA-LoRA.
These results demonstrate that retaining adaptation state benefits
existing TTO methods.
The comparison between VGGT + Causal-Attn and VGGT + StreamTTO
measures the effect of the complete framework, including decoder
adaptation and CSWO, and therefore does not isolate warm-starting.

Read together with the component ablations, these results
highlight the complementary roles of StreamTTO's components.
Causal temporal context strengthens warm-start adaptation,
while CSWO improves the accuracy--computation trade-off
by reusing recent sparse observations.
Freezing the feature extractor enables these observations
to be revisited through cached features, and retaining
decoder parameters and optimizer state carries adaptation
progress across frames.
Their coordinated use makes both visual computation and
learned depth corrections reusable, enabling accurate online
adaptation with few updates per incoming frame.
The component ablations and complete-framework results
jointly support the value of this integrated design
for improving depth accuracy and computational efficiency.
\section{Dataset Specifications}
\label{app:dataset-specifications}

We summarize the acquisition characteristics and temporal coverage of the datasets used in our experiments, followed by the recording-level structure of MS-Depth. Table~\ref{tab:dataset-overview} reports dataset-level statistics, whereas Table~\ref{tab:dataset-evaluation} describes the evaluation subsets.

\subsection{Dataset Overview}

\begin{table*}[t]
\centering
\begingroup
\footnotesize
\resizebox{\linewidth}{!}{%
\begin{tabular}{lllccl}
\toprule
\textbf{Dataset} & \textbf{Environment} & \textbf{Sensing} & \makecell[cc]{\textbf{RGB}\\\textbf{Resolution}} & \textbf{Rate} & \makecell[cl]{\textbf{Total Video Duration /}\\\textbf{Frame Count}} \\
\midrule
KITTI & \makecell[cl]{Outdoor\\Driving} & \makecell[cl]{Stereo RGB;\\64-Channel LiDAR} & \makecell[cc]{Sequence-\\Dependent\footnotemark} & 10 Hz & 1.3 h / 47,937 frames \\
\midrule
Bonn & Indoor & \makecell[cl]{ASUS Xtion\\Pro LIVE RGB-D} & $640\times480$ & 30 Hz & 17 min / 30,022 frames \\
\midrule
NYUv2 Raw & Indoor & Kinect RGB-D & $640\times480$ & 20--30 Hz & 407,024 frames \\
\midrule
TUM RGB-D & Indoor & \makecell[cl]{Kinect /\\ASUS Xtion} & $640\times480$ & 30 Hz & \makecell[cl]{1 h 16 min / 129,824 frames\\(Public Pose GT: 49 min / 81,413 frames;\\Validation: 27 min / 48,411 frames)} \\
\midrule
DDAD & \makecell[cl]{Outdoor\\Driving} & \makecell[cl]{6 RGB Cameras;\\Luminar-H2 LiDAR} & $1936\times1216$ & 10 Hz & \makecell[cl]{28 min / 16,600 frames per camera\\(Train: 21 min / 12,650 frames;\\Validation: 7 min / 3,950 frames)\\Test: 3,080 frames} \\
\midrule
\textbf{MS-Depth} & \makecell[cl]{Indoor /\\Outdoor /\\Underground} & \makecell[cl]{RGB;\\128-Channel LiDAR;\\IMU} & $2048\times1536$ & 10 Hz & \makecell[cl]{40 h / 1.45 M frames\\(Train 36 h / 1.3 M frames;\\Test 4 h / 148,115 frames)} \\
\bottomrule
\end{tabular}%
}
\endgroup
\caption{Dataset specifications and video volume. Available group or split statistics are shown in parentheses. All video durations are approximate, using reported durations or estimates from frame counts and capture rates.}
\label{tab:dataset-overview}
\end{table*}
\footnotetext{KITTI's rectified RGB resolutions are $1224\times370$, $1226\times370$, $1238\times374$, $1241\times376$, and $1242\times375$, as specified in the date-specific calibration files of the official raw release~\citep{geiger2013vision}.}

\paragraph{Public datasets.}
KITTI provides stereo RGB and LiDAR observations for outdoor driving and reports approximately 6 h of data collection~\citep{geiger2013vision}.
However, its public Raw release is limited to sequences with 3D object
annotations or those included in the odometry benchmark's training set.
Consequently, the publicly available footage is substantially shorter
than the reported 6 h.
Direct counts of the public synchronized release yield 47,937 left-camera frames and 47,932 right-camera frames. We report the larger count, corresponding to approximately 1.3 h at 10 Hz.

Bonn provides RGB-D recordings of indoor environments with moving people and objects~\citep{palazzolo2019refusion}. Direct counts of the official release yield 30,022 frames, corresponding to approximately 17 min at the reported capture rate of 30 Hz~\citep{zhang2025ngd}.

NYUv2 Raw contains recordings from 464 indoor scenes~\citep{silberman2012indoor}.
We report its officially listed count of 407,024 unlabeled frames.
Video duration is omitted because the capture rate varies over time and the official documentation specifies a range of 20--30 Hz.

TUM RGB-D supports tracking and SLAM evaluation across indoor sequences with varying motion, scene texture, and dynamics~\citep{sturm2012benchmark}.
Our statistics cover 47 sequences with public camera-trajectory ground truth and 33 validation sequences, excluding nine calibration sequences.
Direct counts of their released RGB recordings yield 81,413 and 48,411 frames, respectively.
The officially listed durations of the 47 sequences sum to approximately 49 min. The validation sequences provide public RGB and depth images, while their reference camera trajectories are withheld and their durations are not listed. We estimate their duration as approximately 27 min from 48,411 recorded frames at 30 Hz. Together, the 80 sequences contain 129,824 frames and cover approximately 1 h 16 min.

DDAD provides camera and LiDAR observations for outdoor driving~\citep{guizilini20203d}. The current public release contains 150 training scenes and 50 validation scenes, each approximately 5~s or 10~s long. 
These splits contain 16,600 frames per camera over approximately 28 min, obtained by summing the reported scene lengths.
Six synchronized cameras capture different views of the same scenes, providing 99,600 frames across all views. This combined frame count remains below the 148,115 frames in the MS-Depth test split.
DDAD also provides 3,080 test images. Their count is listed separately in Table~\ref{tab:dataset-overview}, but they are excluded from its video duration totals because the release documentation does not specify their source video sequences or frame sampling intervals.

\paragraph{MS-Depth.}
All MS-Depth videos are captured from an egocentric perspective using a handheld sensor platform, in contrast to the vehicle-mounted viewpoints of KITTI and DDAD~\citep{geiger2013vision,guizilini20203d}.
Every MS-Depth recording contains at least indoor and outdoor segments within a single uninterrupted video, with some also including underground segments.
Each recording follows a route that returns to its starting point.
These recordings preserve temporal continuity across environmental transitions, enabling evaluation of online adaptation as the observer moves between indoor, outdoor, and underground spaces.
The platform combines high-resolution RGB imagery, 128-channel LiDAR, and IMU measurements. The RGB camera has a resolution of $2048\times1536$, and the Ouster OS-1-128-SR LiDAR operates at 10 Hz.
A trigger board synchronizes sensor acquisition. The training split contains 107 recordings, approximately 1.3 million frames, and approximately 36 h of video.
The test split contains 13 recordings, 148,115 frames, and approximately 4 h of video. Our online depth-completion evaluation uses the test split.

\subsection{Evaluation Subsets}

\begin{table*}[t]
\centering
\begin{minipage}[c]{0.4\linewidth}
\centering
\begingroup
\footnotesize
\resizebox{\linewidth}{!}{%
\begin{tabular}{llr}
\toprule
\textbf{Dataset} & \makecell[cl]{\textbf{Evaluation}\\\textbf{Subset}} & \textbf{Frames} \\
\midrule
KITTI & \makecell[cl]{8 Sequences\\$\times$ 110 Frames} & 880 \\
\midrule
Bonn & \makecell[cl]{5 Sequences\\$\times$ 110 Frames} & 550 \\
\midrule
NYUv2 Raw & \makecell[cl]{8 Sequences\\$\times$ 110 Frames} & 880 \\
\midrule
TUM RGB-D & \makecell[cl]{8 Sequences\\$\times$ 110 Frames} & 880 \\
\midrule
DDAD & \makecell[cl]{50 Val. Scenes\\(Front Camera)} & 3,950 \\
\midrule
\textbf{MS-Depth} & \makecell[cl]{13 Test\\Recordings} & 148,115 \\
\bottomrule
\end{tabular}%
}
\endgroup
\caption{Evaluation subsets and frame counts used in our experiments. For each dataset, we report the selected sequences or scenes and the corresponding total frame count. The MS-Depth count includes TE04 for development and the remaining 12 sequences for evaluation.}
\label{tab:dataset-evaluation}
\end{minipage}\hfill
\begin{minipage}[c]{0.58\linewidth}
\centering
\raisebox{0pt}[\height][0pt]{%
\includegraphics[width=\linewidth]{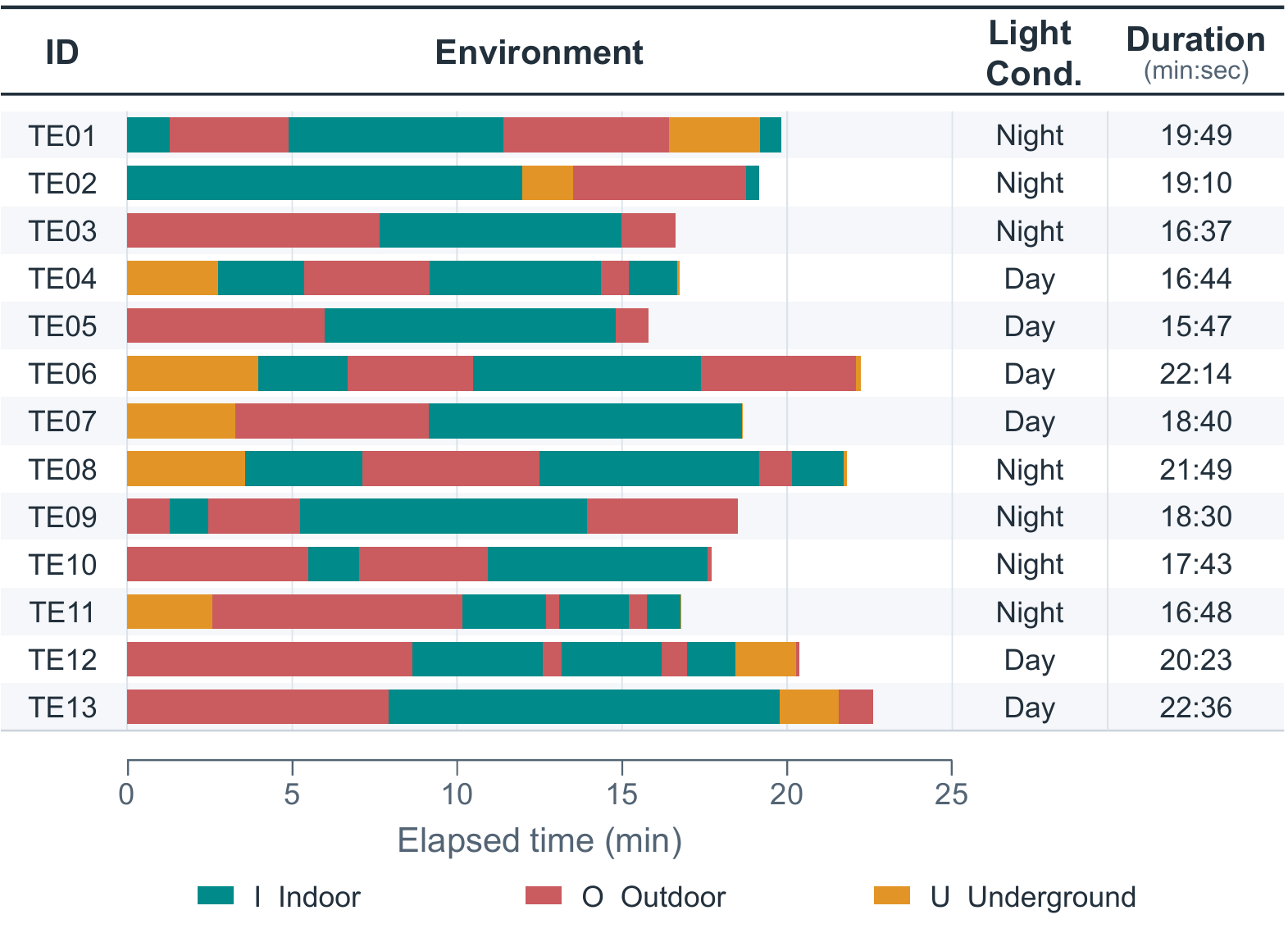}%
}
\captionof{figure}{Environment transitions in the MS-Depth test recordings.
Each bar shows one continuous recording, with colors indicating environment types and segment widths representing durations on a shared time axis.}
\label{fig:msdepth-environment-timeline}
\end{minipage}
\end{table*}

As described in Section~\ref{sec:experimental-setup}, we construct sparse depth inputs for KITTI, DDAD, and MS-Depth using a fixed selection of LiDAR scan rings within each sequence, retaining at least 8 rings.
For Bonn, NYUv2 Raw, and TUM RGB-D, we sample 10\% of valid reference-depth pixels.
All methods using sparse depth receive the same observations, and evaluation excludes the observed input locations.

\subsection{Temporal Structure of MS-Depth Test Set}

\paragraph{Recording coverage.}
The 13 test recordings cover eight labeled locations/routes.
The recordings include a mix of daytime and nighttime conditions.
Five location/route labels appear under two different capture conditions, supporting comparisons across times of day within the same labeled location or route.
Figure~\ref{fig:msdepth-environment-timeline} shows the environment transitions and duration of each test recording.
Individual recordings last approximately 16--23 min.

\paragraph{Environment transitions.}
Every test recording contains both indoor and outdoor segments, and nine also contain underground segments. The annotations partition all retained frames into 70 contiguous environment segments without gaps or overlaps.
These segments define 57 within-recording transitions, where a transition is a change in the environment label between adjacent segments.
The directed transition counts are 18 indoor-to-outdoor, 21 outdoor-to-indoor, 7 indoor-to-underground, 4 underground-to-indoor, 2 outdoor-to-underground, and 5 underground-to-outdoor.
Each recording returns to its starting point after traversing multiple environment types.

\paragraph{Implications for online adaptation.}
MS-Depth enables evaluation of sparse-depth-guided adaptation along continuous recordings with repeated changes between indoor, outdoor, and underground environments.
Maintaining the adaptation state across these transitions makes it possible to examine whether updates remain useful after a scene change, whether errors introduced by earlier updates persist, and how predictions recover when an environment type reappears.
The segment annotations support evaluation before and after transitions and upon returning to an earlier environment type.

\subsection{Comparison with SLAM and 3D Reconstruction Datasets}
\label{sec:comparison_slam_reconstruction}

In this subsection, we compare MS-Depth with FusionPortableV2~\citep{wei2025fusionportablev2} and Oxford Spires~\citep{tao2026oxford}, which also provide calibrated RGB–LiDAR recordings with indoor–outdoor transitions, to clarify their differences in supporting long-horizon online adaptation for depth completion.

FusionPortableV2 (FPv2) targets SLAM generalization across platforms, and Oxford Spires benchmarks localization, 3D reconstruction, and novel-view synthesis.
Both provide calibrated RGB--LiDAR recordings across varied environments, while MS-Depth is designed to evaluate long-horizon online adaptation guided by sparse depth across environmental changes. 
Table~\ref{tab:slam_reconstruction_datasets} summarizes their acquisition settings and adaptation-relevant characteristics.
We report sequences lasting at least 15 minutes to match the final elapsed-time interval in our long-horizon adaptation analysis.
Oxford recording durations were computed from front-camera image timestamps.

\begin{table*}[t]
\centering
\begingroup
\footnotesize
\setlength{\tabcolsep}{2.5pt}
\renewcommand{\arraystretch}{1.17}
\resizebox{\linewidth}{!}{%
\begin{tabular}{@{}lcccccccccc@{}}
\toprule
\textbf{Dataset} & \textbf{Platform} & \makecell[cc]{\textbf{RGB}\\\textbf{resolution}} & \textbf{Rate} & \makecell[cc]{\textbf{LiDAR}\\\textbf{channels}} & \makecell[cc]{\textbf{Duration}\\\textbf{(min/seq.)}} & \makecell[cc]{\textbf{Seq.}\\\textbf{\boldmath$\geq$15\,min}} & \textbf{I/O seq.} & \makecell[cc]{\textbf{Day /}\\\textbf{Night}} & \makecell[cc]{\textbf{Depth}\\\textbf{data}} & \makecell[cc]{\textbf{Env.}\\\textbf{info.}} \\
\midrule
FusionPortableV2 & Multiple & $1024\times768$ & 20 Hz & 128 & 1.3--11.6 & 0/27 & 6/27 & 24 / 3 & \makecell[cc]{Subset +\\tools} & Sequence \\
\midrule
Oxford Spires & \makecell[cc]{Human-\\carried} & $1440\times1080$ & 5--20 Hz & 64 & 3.5--21.3 & 3/24 & 11/24 & 24 / 0 & \makecell[cc]{Subset +\\tools} & Sequence \\
\midrule
MS-Depth (test) & Handheld & $2048\times1536$ & 10 Hz & 128 & 15.8--22.6 & 13/13 & 13/13 & 6 / 7 & \makecell[cc]{Input +\\reference} & Temporal \\
\bottomrule
\end{tabular}%
}
\endgroup
\caption{Dataset characteristics relevant to continuous adaptation.
Duration is per recording. Fractions denote qualifying/total sequences;
Day/Night gives capture-time counts. I/O denotes indoor--outdoor recordings.
MS-Depth uses the test split.}
\label{tab:slam_reconstruction_datasets}
\end{table*}

The FPv2 paper identifies short recording durations as a limitation, with all recordings lasting at most 11.6 minutes and the six handheld recordings spanning 1.9--6.3 minutes.
Oxford Spires includes longer raw recordings, but only 3 of 24 last at least 15 minutes, limiting the number of sequences available for long-horizon evaluation.
To support such evaluation across a larger set of sequences, we constructed the MS-Depth test split with 13 recordings lasting 15.8--22.6 minutes each.

FPv2 and Oxford Spires contain indoor--outdoor transitions in a subset of their recordings.
Because FPv2's official release lacks explicit I/O sequence labels, we inspected its public RGB previews and identified six I/O recordings.
We included parking interiors and enclosed road tunnels but excluded open-sided canopies.
All six were captured using legged robots, UGVs, or vehicles.
For Oxford Spires, six recordings are officially tagged as both Outdoor and Indoor (proper), and eleven as both Outdoor and Indoor (limited). The latter tag indicates short indoor coverage or challenging lighting.
Inspection of all 11 limited recordings identified 5 with indoor scenes, excluding open-to-sky spaces and very short tunnel-like passageways as ambiguous cases.
Together with the 6 proper recordings, these yield 11/24 I/O recordings.
Neither release provides explicit day/night labels.
We therefore used recording timestamps and RGB inspection, identifying daytime and nighttime recordings in FPv2 and only daytime recordings in Oxford Spires.
MS-Depth includes indoor--outdoor transitions in every handheld test recording and provides separate daytime and nighttime recordings of shared locations.
Beyond recording-level environment descriptions and tags, MS-Depth provides temporal labels for analyzing adaptation across transitions.

FPv2 and Oxford Spires, primarily developed for SLAM and 3D reconstruction, provide precomputed depth images for selected recordings; others require processing the released RGB--LiDAR data and calibration with the provided tools.
In contrast, MS-Depth is captured with a 128-channel LiDAR and configured for depth completion: only eight selected channels are provided as sparse input, and depth at unobserved pixels must be predicted.
For every recording, we provide calibrated RGB--depth pairs with precomputed input and reference depth, alongside the raw sensor data.

\section{Evaluation Subset Construction for Public Benchmarks}
\label{app:evaluation_subset_construction}

We construct temporally continuous RGB--depth evaluation subsets from
KITTI, Bonn, NYUv2 Raw, TUM RGB-D, and DDAD.
For KITTI, Bonn, NYUv2 Raw, and TUM RGB-D, we use one
110-pair clip from each selected sequence.
This choice standardizes the clip length across these four datasets
under a fixed evaluation budget.
The dataset-specific sequence-selection and clip-extraction rules are
described below.
Clip extraction uses data availability, temporal continuity, and
RGB--depth association criteria, without using model predictions or
evaluation errors.
For DDAD, we use every front-camera frame from all 50 scenes in the
official validation split.
The resulting public-benchmark evaluation comprises 79 sequences and
7,140 RGB--depth pairs (Table~\ref{tab:public_subset_counts}).
All methods are evaluated on the same selected pairs, and each pair
contributes to the accuracy evaluation exactly once.

\begin{table}[t]
\centering
\small
\setlength{\tabcolsep}{6pt}
\renewcommand{\arraystretch}{1.08}
\begin{tabular}{@{}lrrr@{}}
\toprule
Dataset & Sequences & Pairs per sequence & Total pairs \\
\midrule
KITTI       & 8  & 110       & 880   \\
Bonn        & 5  & 110       & 550   \\
NYUv2 Raw   & 8  & 110       & 880   \\
TUM RGB-D   & 8  & 110       & 880   \\
DDAD        & 50 & 50 or 100 & 3,950 \\
\midrule
Total       & 79 &          & 7,140 \\
\bottomrule
\end{tabular}
\caption{Public-benchmark evaluation subsets.
Each sequence contributes one evaluation clip; DDAD clips contain
all front-camera frames from their respective validation scenes.}
\label{tab:public_subset_counts}
\end{table}

\paragraph{Frame identification and interval boundaries.}
All start and end frames listed below are included in the evaluation.
For Bonn, NYUv2 Raw, and TUM RGB-D, we identify evaluated pairs by
their depth timestamps.
The timestamp bounds specify the first and last selected depth frames;
the exact RGB--depth associations are recorded in frame-level manifests.
KITTI retains its original frame IDs, while DDAD retains the original
timestamp tokens together with the within-scene sample order.

\paragraph{KITTI.}
We examine all 13 drives in the official KITTI depth-completion
validation split and retain all eight drives that provide at least
110 consecutive frame IDs with both RGB images and corresponding
annotated depth maps (Table~\ref{tab:kitti_evaluation_clips}).
The remaining five drives do not contain a sufficiently long
consecutive RGB--depth interval.
For each eligible drive, we identify the longest consecutive run of
frame IDs shared by the RGB images and annotated depth maps.
If multiple runs have the same maximum length, we select the earliest one.
We then use the first 110 frames of the selected run.
For all eight drives, the resulting interval is
\texttt{0000000005}--\texttt{0000000114}, inclusive.

\begin{table}[t]
\centering
\small
\setlength{\tabcolsep}{4pt}
\renewcommand{\arraystretch}{1.08}
\begin{tabular}{@{}lrrr@{}}
\toprule
Sequence ID & Start frame & End frame & Pairs \\
\midrule
\texttt{2011\_09\_26\_drive\_0005\_sync}
& \texttt{0000000005} & \texttt{0000000114} & 110 \\
\texttt{2011\_09\_26\_drive\_0013\_sync}
& \texttt{0000000005} & \texttt{0000000114} & 110 \\
\texttt{2011\_09\_26\_drive\_0023\_sync}
& \texttt{0000000005} & \texttt{0000000114} & 110 \\
\texttt{2011\_09\_26\_drive\_0036\_sync}
& \texttt{0000000005} & \texttt{0000000114} & 110 \\
\texttt{2011\_09\_26\_drive\_0095\_sync}
& \texttt{0000000005} & \texttt{0000000114} & 110 \\
\texttt{2011\_09\_29\_drive\_0026\_sync}
& \texttt{0000000005} & \texttt{0000000114} & 110 \\
\texttt{2011\_09\_30\_drive\_0016\_sync}
& \texttt{0000000005} & \texttt{0000000114} & 110 \\
\texttt{2011\_10\_03\_drive\_0047\_sync}
& \texttt{0000000005} & \texttt{0000000114} & 110 \\
\bottomrule
\end{tabular}
\caption{KITTI evaluation sequences and inclusive frame-ID intervals.}
\label{tab:kitti_evaluation_clips}
\end{table}

\paragraph{Bonn.}
We adopt the five-sequence Bonn evaluation list used by Video Depth
Anything (VDA) and DepthCrafter, as listed in
Table~\ref{tab:bonn_evaluation_clips}.
Within each sequence, we select the earliest 110-frame depth interval
in which the time between successive depth frames does not exceed
0.05~s and every depth frame has an RGB match satisfying the
association threshold.
Each depth frame is paired with the temporally nearest RGB frame,
with an absolute RGB--depth timestamp difference of at most 0.025~s.

\begin{table}[t]
\centering
\small
\setlength{\tabcolsep}{4pt}
\renewcommand{\arraystretch}{1.08}
\begin{tabular}{@{}lrrr@{}}
\toprule
Sequence ID & Start depth timestamp & End depth timestamp & Pairs \\
\midrule
\texttt{rgbd\_bonn\_synchronous}
& 1548340301.91798 & 1548340305.55515 & 110 \\
\texttt{rgbd\_bonn\_person\_tracking}
& 1548265882.51176 & 1548265886.14241 & 110 \\
\texttt{rgbd\_bonn\_crowd2}
& 1548339894.59618 & 1548339898.23325 & 110 \\
\texttt{rgbd\_bonn\_crowd3}
& 1548339959.43347 & 1548339963.07062 & 110 \\
\texttt{rgbd\_bonn\_balloon2}
& 1548266533.17712 & 1548266536.81403 & 110 \\
\bottomrule
\end{tabular}
\caption{Bonn evaluation sequences and inclusive depth-timestamp intervals.}
\label{tab:bonn_evaluation_clips}
\end{table}

\paragraph{NYUv2 Raw.}
We use the eight Raw video sequences specified by the official VDA
evaluation helper, as listed in Table~\ref{tab:nyuv2_evaluation_clips}.
Within each sequence, we select the earliest 110-frame depth interval
in which successive depth timestamps differ by at most 0.05~s
and every depth frame has an RGB match within 0.025~s.
Each depth frame is paired with its temporally nearest RGB frame.
This evaluation uses the specified Raw video clips and is distinct from
evaluation on the 654-image NYUv2 test split.

\begin{table}[t]
\centering
\small
\setlength{\tabcolsep}{4pt}
\renewcommand{\arraystretch}{1.08}
\begin{tabular}{@{}lrrr@{}}
\toprule
Sequence ID & Start depth timestamp & End depth timestamp & Pairs \\
\midrule
\texttt{basement\_0001a}
& 1316653609.773176 & 1316653613.405163 & 110 \\
\texttt{bookstore\_0001a}
& 1295525659.331413 & 1295525662.963355 & 110 \\
\texttt{cafe\_0001a}
& 1295527093.873179 & 1295527097.504396 & 110 \\
\texttt{classroom\_0001a}
& 1315404991.101447 & 1315404994.732725 & 110 \\
\texttt{kitchen\_0003}
& 1294887072.518566 & 1294887076.168351 & 110 \\
\texttt{office\_0004}
& 1294851491.914890 & 1294851495.546016 & 110 \\
\texttt{playroom\_0002}
& 1315169148.541866 & 1315169152.154410 & 110 \\
\texttt{study\_0002}
& 1315161478.295818 & 1315161481.920261 & 110 \\
\bottomrule
\end{tabular}
\caption{NYUv2 Raw evaluation sequences and inclusive depth-timestamp intervals.}
\label{tab:nyuv2_evaluation_clips}
\end{table}

\paragraph{TUM RGB-D.}
We use a custom subset of eight sequences, listed in
Table~\ref{tab:tum_evaluation_clips}, fixed before evaluating methods
on TUM RGB-D.
The subset includes sequences from Freiburg~1, 2, and 3.
RGB and depth frames are associated using the official TUM
greedy one-to-one timestamp-matching procedure with a maximum
absolute timestamp difference of 0.02~s.
After association, we select the earliest interval of 110 matched pairs
for which successive timestamps differ by at most 0.075~s in both
the matched RGB stream and the matched depth stream.

\begin{table}[t]
\centering
\scriptsize
\setlength{\tabcolsep}{3pt}
\renewcommand{\arraystretch}{1.12}
\begin{tabular}{@{}lrrr@{}}
\toprule
Sequence ID & Start depth timestamp & End depth timestamp & Pairs \\
\midrule
\texttt{rgbd\_dataset\_freiburg1\_desk}
& 1305031453.374112 & 1305031457.342954 & 110 \\
\texttt{rgbd\_dataset\_freiburg1\_desk2}
& 1305031526.688356 & 1305031530.656445 & 110 \\
\texttt{rgbd\_dataset\_freiburg1\_room}
& 1305031910.771502 & 1305031914.407339 & 110 \\
\texttt{rgbd\_dataset\_freiburg2\_desk}
& 1311868164.373557 & 1311868168.115940 & 110 \\
\texttt{rgbd\_dataset\_freiburg2\_desk\_with\_person}
& 1311870427.207687 & 1311870431.247317 & 110 \\
\texttt{rgbd\_dataset\_freiburg2\_large\_with\_loop}
& 1311875572.417883 & 1311875576.219311 & 110 \\
\texttt{rgbd\_dataset\_freiburg3\_long\_office\_household}
& 1341847980.723020 & 1341847984.542793 & 110 \\
\texttt{rgbd\_dataset\_freiburg3\_walking\_xyz}
& 1341846313.592088 & 1341846317.442587 & 110 \\
\bottomrule
\end{tabular}
\caption{TUM RGB-D evaluation sequences and inclusive depth-timestamp intervals.}
\label{tab:tum_evaluation_clips}
\end{table}

\paragraph{DDAD.}
We use all 50 scenes in the official validation split,
\texttt{000150}--\texttt{000199}, and evaluate all frames from
the front camera, \texttt{CAMERA\_01}, in their original sample order.
The 21 scenes listed in the first row of
Table~\ref{tab:ddad_evaluation_scenes} contain 50 frames each,
with zero-based sample indices 0--49.
The remaining 29 scenes contain 100 frames each,
with sample indices 0--99.
This gives $21\times50+29\times100=3{,}950$ evaluated RGB--depth pairs.

\begin{table}[t]
\centering
\small
\setlength{\tabcolsep}{4pt}
\renewcommand{\arraystretch}{1.12}
\begin{tabular}{@{}rrp{0.68\linewidth}@{}}
\toprule
\shortstack[r]{Frames\\per scene} &
\shortstack[r]{Sample\\indices} & Scene IDs \\
\midrule
50 & 0--49 &
\texttt{000151, 000152, 000154, 000157, 000159, 000162, 000166,
000172, 000174, 000177, 000182, 000183, 000191, 000192,
000193, 000194, 000195, 000196, 000197, 000198, 000199} \\
\addlinespace
100 & 0--99 &
\texttt{000150, 000153, 000155, 000156, 000158, 000160, 000161,
000163, 000164, 000165, 000167, 000168, 000169, 000170,
000171, 000173, 000175, 000176, 000178, 000179, 000180,
000181, 000184, 000185, 000186, 000187, 000188, 000189,
000190} \\
\bottomrule
\end{tabular}
\caption{Complete DDAD validation-scene coverage for \texttt{CAMERA\_01}.
Sample-index intervals are zero-based and inclusive.}
\label{tab:ddad_evaluation_scenes}
\end{table}

\section{Analysis of Cached Online Decoder Adaptation}
\label{sec:cached_adaptation_analysis}

We analyze how representation caching and overlapping observation windows
support efficient online decoder adaptation.
Building on causal KV caching and decoder-only test-time adaptation,
we formalize two properties of their use in StreamTTO:
the consistency of cached representations during adaptation,
and the change in the backpropagated gradient when
the observation window advances.
The first property concerns the cost of repeated optimization; the second
provides a conditional explanation for reusing optimization progress.

\paragraph{Setup.}
Let $\phi$ denote the frozen feature-extractor parameters and $\theta$
the adapted decoder parameters.
We assume deterministic feature extraction, causal temporal operations,
and fixed positional assignments.
The attention mask and KV retention policy for a frame depend only on
information available when that frame arrives.
The decoder processes each cached frame independently, with parameters
shared across frames and no batch-dependent coupling.
All comparisons below use the same input history, initial cache state,
and causal computation policy.

\paragraph{Proposition 1 (Consistency of cached representations).}
Let $Z_j^{[t]}$ denote the representation at frame $j$ in a reference
execution on the prefix $I_{1:t}$ under the preceding assumptions.
For every $t\geq j$, and for any subsequent decoder parameter updates,
\begin{equation}
    Z_j^{[t]} = Z_j^{[j]} \equiv Z_j,
    \qquad
    \frac{\partial Z_j}{\partial\theta}=0.
    \label{eq:streamtto_cache_consistency}
\end{equation}
Consequently, evaluating the decoder and its sparse-depth loss using
cached $Z_j$ gives the same result as recomputing that representation
under the same reference execution.

\noindent\emph{Proof.}
At each causal attention layer, the representation of frame $j$ depends
only on states available at or before $j$.
Appending frames with indices greater than $j$ therefore does not change
its attention inputs or mask.
Induction over the layers establishes the first equality.
Since the extractor is frozen and has no dependence on decoder
parameters, decoder updates leave $Z_j$ unchanged, establishing the
second equality.
Substituting the same $Z_j$ into the decoder and loss gives identical
mathematical computations at any fixed $\theta$.\hfill$\square$

This result justifies reusing temporally contextualized features during
decoder adaptation without repeated feature extraction.
The reference execution uses the same causal KV policy as streaming
inference; it is not a fresh execution on an isolated sliding clip with
its earlier state discarded.
Updates to parameters that affect cached representations generally
require refreshing those representations to match a replay under the
updated parameters.
Decoder adaptation preserves cache consistency structurally, although
other updates that leave all cached quantities unchanged can also do so.

\paragraph{Window gradients under the implemented loss.}
For a retained pair $(Z_j,S_j)$, let $g_j(\theta)$ denote the gradient
backpropagated to the decoder from the frame-wise sparse-depth loss,
before global norm clipping and the AdamW update.
As in the main method, scale and shift are recomputed at each loss
evaluation and held constant during backpropagation.
Thus, $g_j$ follows the implemented stop-gradient convention; it need not
equal the total derivative through the alignment solver.
Decoder predictions, alignment coefficients, and gradients are
reevaluated at the current $\theta$; only the features and observations
are reused.

Once the window is full, let its size be $W\geq2$ and assign weight
$\alpha\geq1$ to the current frame and weight one to each past frame.
Writing $c=W-1+\alpha$, the normalized window gradient is
\begin{equation}
    G_t(\theta)
    =\frac{1}{c}\left[
        \alpha g_t(\theta)
        +\sum_{j=t-W+1}^{t-1}g_j(\theta)
    \right].
    \label{eq:streamtto_window_gradient}
\end{equation}
This expression uses the same frame-wise loss normalization as the main
method and therefore remains valid when observation counts differ
between frames.

\paragraph{Proposition 2 (Change across overlapping windows).}
For two consecutive full windows, $t\geq W+1$, evaluated at the same
decoder parameters $\theta$,
\begin{equation}
    G_t(\theta)-G_{t-1}(\theta)
    =\frac{
        \alpha g_t(\theta)
        -(\alpha-1)g_{t-1}(\theta)
        -g_{t-W}(\theta)
    }{c}.
    \label{eq:streamtto_gradient_change}
\end{equation}
In particular, define
\begin{equation}
    \delta_{a,b}(\theta)
    =\left\|g_a(\theta)-g_b(\theta)\right\|_2.
    \label{eq:streamtto_gradient_difference}
\end{equation}
Then
\begin{equation}
    \left\|G_t(\theta)-G_{t-1}(\theta)\right\|_2
    \leq
    \frac{
        \alpha\delta_{t,t-1}(\theta)
        +\delta_{t-1,t-W}(\theta)
    }{c}.
    \label{eq:streamtto_gradient_change_bound}
\end{equation}

\noindent\emph{Proof.}
Proposition 1 makes each retained pair $(Z_j,S_j)$ independent of the
subsequent window arrival time.
At a common $\theta$, its backpropagated gradient is therefore the same
in both windows.
Subtracting their normalized sums cancels the terms with indices
$t-W+1,\ldots,t-2$.
The remaining terms correspond to the new frame, the previous current
frame whose weight changes from $\alpha$ to one, and the removed frame.
Rewriting their numerator as
$\alpha(g_t-g_{t-1})+(g_{t-1}-g_{t-W})$ and applying the triangle
inequality proves the bound.\hfill$\square$

\paragraph{Implication for warm-start adaptation.}
Let $\theta_{t-1}$ be the decoder state inherited from the previous
frame, and set
$\varepsilon_{t-1}=\|G_{t-1}(\theta_{t-1})\|_2$.
Proposition 2 gives
\begin{equation}
    \begin{aligned}
    \left\|G_t(\theta_{t-1})\right\|_2
    \leq {}& \varepsilon_{t-1} \\
    &+\frac{
        \alpha\delta_{t,t-1}(\theta_{t-1})
        +\delta_{t-1,t-W}(\theta_{t-1})
    }{c}.
    \end{aligned}
    \label{eq:streamtto_warm_start_bound}
\end{equation}
If previous adaptation has reduced the preceding window gradient and
the relevant frame gradients vary little, the inherited decoder also
has a small backpropagated gradient on the new window.
The default weights $(W,\alpha)=(3,2)$ yield
\begin{equation}
    G_t(\theta)-G_{t-1}(\theta)
    =\frac{
        2g_t(\theta)-g_{t-1}(\theta)-g_{t-3}(\theta)
    }{4}.
    \label{eq:streamtto_default_gradient_change}
\end{equation}
For $\alpha=1$, the change instead reduces to
$(g_t-g_{t-W})/W$, covering the equal-weight diagnostic experiments.
The identity concerns gradients evaluated at a common parameter state;
the gradients themselves are not cached across optimizer updates.

\paragraph{Empirical gradient diagnostics.}
We examine the gradient behavior underlying warm-start adaptation
on the MS-Depth TE04 development sequence with window sizes
$W\in\{1,3,6,12\}$, keeping the remaining adaptation settings fixed.
For each incoming frame, we evaluate gradients before its first
adaptation update.
Let $\theta_{t-1}$ denote the inherited decoder parameters and
$\theta_{\mathrm{pre}}$ the pretrained decoder parameters.
Using the implemented sparse-depth loss, we measure
\begin{align*}
q_t &=
\frac{\|G_t(\theta_{t-1})\|_2}
     {\|G_t(\theta_{\mathrm{pre}})\|_2+\epsilon},\\
d_t &=
\|G_t(\theta_{t-1})-G_{t-1}(\theta_{t-1})\|_2,\\
s_t &=
\operatorname{cos}\!\left(
G_t(\theta_{t-1}),G_{t-1}(\theta_{t-1})
\right),
\end{align*}
where $\epsilon>0$ is a small constant for numerical stability.
The ratio $q_t$ compares the two initializations on the same current
window using identical cached features, sparse observations, and
loss weights.
Both gradients used in $d_t$ and $s_t$ are reevaluated at the same
inherited parameters $\theta_{t-1}$.
All gradients are measured before clipping.
We aggregate statistics over the same post-bootstrap frame indices
for which both consecutive windows are full for every tested $W$.
The $W=1$ setting provides a single-frame reference.

\begin{table}[t]
\centering
\small
\setlength{\tabcolsep}{6pt}
\renewcommand{\arraystretch}{1.1}
\begin{tabular}{@{}lccccc@{}}
\toprule
$W$
& \shortstack{Mean\\$q_t$}
& \shortstack{Median\\$q_t$}
& \shortstack{Frames with\\$q_t<1$}
& \shortstack{Mean\\$d_t$}
& \shortstack{Mean\\$s_t$} \\
\midrule
1
& 0.58 & 0.30 & 84\% & 13.9 & 0.55 \\

3 (default)
& 0.55 & 0.29 & 86\% & 9.4 & 0.78 \\

6
& 0.57 & 0.30 & 85\% & 6.1 & 0.86 \\

12
& 1.05 & 0.46 & 70\% & 4.3 & 0.91 \\
\bottomrule
\end{tabular}
\caption{
Decoder-gradient diagnostics on MS-Depth TE04 before current-frame
adaptation.
The ratio $q_t$ compares warm-start and pretrained initializations,
while $d_t$ and $s_t$ measure gradient difference and cosine
similarity between consecutive windows evaluated at the same
decoder parameters.
}
\label{tab:gradient_diagnostics}
\end{table}

At the default window size $W=3$, the mean and median gradient norm
ratios are 0.55 and 0.29, respectively, and warm-start initialization
yields a smaller gradient on 86\% of frames.
Across the tested window sizes, the mean gradient difference
decreases from 13.9 to 4.3, while the mean cosine similarity
increases from 0.55 to 0.91.
Thus, larger windows exhibit smaller gradient changes and greater
directional alignment between consecutive objectives in these runs.

The warm-start gradient norm ratio does not improve monotonically
with window size.
At $W=12$, its median remains below one at 0.46 and 70\% of frames
have $q_t<1$, but its mean increases to 1.05.
This distinguishes consistency between consecutive window gradients
from the reduction obtained relative to pretrained initialization.
The result is consistent with the conditional analysis: a small
change between windows alone does not ensure a small gradient on
the new window, which also depends on the preceding window gradient
at the inherited decoder state.

\paragraph{Computational implications.}
Let $C_{\mathcal F}$ be the cost of extracting the incoming frame's
features with the fixed KV context, $C_{\mathrm{upd}}(W)$ the cost of
one complete window update, including decoder forward and backward
computation, alignment, clipping, and the optimizer step, and
$C_{\mathrm{out}}$ the cost of producing the final current-frame output.
After bootstrap, the model-computation cost is
\begin{equation}
    C_{\mathrm{frame}}(W,R)
    =C_{\mathcal F}
    +R\,C_{\mathrm{upd}}(W)
    +C_{\mathrm{out}}.
    \label{eq:streamtto_cached_update_cost}
\end{equation}
Feature extraction is paid once per incoming frame and is not multiplied
by the number of optimization passes.
With fixed KV context, $W$, input resolution, and model size, both
per-frame model computation and persistent state memory are bounded
independently of the total stream length.

\paragraph{Interpretation and scope.}
The analysis separates exact representation reuse from the conditional
reuse of optimization progress.
Causality and freezing make cached features consistent, while window
overlap permits cancellation of unchanged gradient terms.
Temporal continuity alone does not guarantee that the gradient
differences in ~\eqref{eq:streamtto_warm_start_bound} are small,
particularly at environmental transitions.
The bound concerns the backpropagated gradient and does not establish
closeness to a minimizer, a bound on AdamW parameter displacement, or
convergence in a prescribed number of updates.
Those claims would require additional assumptions and analysis of the
optimizer dynamics.

The empirical ablations assess whether the proposed reuse is effective
under the evaluated conditions.
The initialization ablation shows that two warm-start passes achieve
accuracy comparable to twenty reset passes.
The compute-matched transition experiment further shows a 6.6\%
reduction in post-transition RMSE when recent observations are reused
in a three-frame window.
These results support the practical choice of a small update budget;
the analysis provides its structural and conditional rationale without
asserting that $R=2$ is universally sufficient or optimal.

\end{document}